\documentclass[twoside]{article}

\usepackage[accepted]{aistats2023}
\usepackage{amsmath}
\usepackage{amssymb}
\usepackage{hyperref}
\usepackage{stmaryrd}

\usepackage{microtype}
\usepackage{graphicx}
\usepackage{subfigure}
\usepackage{booktabs} 
\usepackage{cases}
\usepackage{mathtools}

\usepackage{amsthm}
\newtheorem{definition}{Definition}

\usepackage{flushend}
\newtheorem{lem}{Lemma}
\newtheorem{prop}{Proposition}
\newtheorem{theorem}{Theorem}

\usepackage{xcolor}

\begin{document}

%

%

\twocolumn[

\aistatstitle{On Model Selection Consistency of Lasso for High-Dimensional Ising Models}

\aistatsauthor{Xiangming Meng \And Tomoyuki Obuchi \And  Yoshiyuki Kabashima}

\aistatsaddress{The University of Tokyo \And   Kyoto University \And The University of Tokyo } ]

\begin{abstract}
We theoretically analyze the model selection consistency of least absolute shrinkage and selection operator (Lasso), {\textit{both with and without post-thresholding}}, for high-dimensional Ising models. For random regular (RR) graphs of size $p$ with regular  node degree $d$ and uniform couplings $\theta_0$, it is rigorously proved that Lasso \textit{without post-thresholding} is model selection consistent in the  whole paramagnetic phase with the same  order of sample complexity $n=\Omega{(d^3\log{p})}$ as that of $\ell_1$-regularized logistic regression ($\ell_1$-LogR). This result is consistent with the conjecture in  \textit{Meng, Obuchi, and Kabashima 2021} \cite{meng2021ising} using the non-rigorous replica method from statistical physics  and thus complements it with a rigorous proof. For general tree-like graphs, it is demonstrated that the same result as RR graphs can be obtained under mild assumptions of the dependency condition and incoherence condition. Moreover, we provide a rigorous proof of the model selection consistency of Lasso \textit{with post-thresholding} for general tree-like graphs in the paramagnetic phase without further assumptions on the dependency and incoherence conditions. 
Experimental results agree well with our theoretical analysis.
\end{abstract}

\section{Introduction}
Ising model \cite{ising1925beitrag} is one renowned binary undirected graphical models (also known
as Markov random fields (MRFs)) \cite{wainwright2008graphical,koller2009probabilistic, mezard2009information} with wide applications in various scientific disciplines such as social networking \cite{mcauley2012learning}, gene network analysis \cite{marbach2012wisdom,krishnan2020modified}, and protein interactions \cite{morcos2011direct,liebl2021accurate}, just to name a few. Given an undirected graph $G=\left(V,E\right)$,
where $V=\left\{ 1,...,p\right\} $ is a collection of nodes associated with the binary spins $X=\left(X_{i}\right)_{i=1}^{p}$ and $E=\left\{ \left(r,t\right)|\theta_{rt}^{*}\neq0\right\} $
is a collection of undirected edges that specify the pairwise interactions $\theta^{*}=\left(\theta_{rt}^{*}\right)_{r\neq t}$, the joint probability distribution 
of an Ising model  
has the following form
\begin{equation}
\mathbb{P}_{\theta^{*}}\left(x\right)=\frac{1}{Z\left(\theta^{*}\right)}\exp\big\{ \sum_{r\neq t}\theta_{rt}^{*}x_{r}x_{t}\big\} ,\label{eq:Ising_distribution}
\end{equation}
where $Z\left(\theta^{*}\right)=\sum_{x}\exp\left\{ \sum_{r\neq t}\theta_{rt}^{*}x_{r}x_{t}\right\}$
is the partition function. In general, there are also external
fields but here they are assumed to be zero for simplicity. Importantly, the conditional independence between $X=\left(X_{i}\right)_{i=1}^{p}$ can be well captured by the associated graph $G$ \cite{wainwright2008graphical,koller2009probabilistic} and hence one fundamental problem, namely  Ising model selection {(also widely known as the inverse Ising problem.  Please refer to \cite{nguyen2017inverse} for a nice review)}, is to recover the underlying graph structure
(edge set $E$) of $G$ from a collection of $n$ i.i.d. samples $\mathfrak{X}_n\coloneqq\left\{ x^{(1)},\ldots,x^{(n)}\right\} $, where $x^{(i)}\in\left\{ -1,+1\right\} ^{p}$
represents the $i$-th sample.
To address this fundamental problem, a variety of methods have been proposed over the past several decades in various fields \cite{tanaka1998mean,kappen1998efficient,ricci2012bethe, wainwright2007high,hofling2009estimation,ravikumar2010high,decelle2014pseudolikelihood,bresler2015efficiently,vuffray2016interaction,lokhov2018optimal}. Notably, under the framework of the pseudo-likelihood (PL) \cite{besag1975statistical}, both $\ell_{1}$-regularized logistic regression ($\ell_{1}$-LogR) \cite{ravikumar2010high} and $\ell_{1}$-regularized interaction screening estimator (RISE) \cite{vuffray2016interaction,lokhov2018optimal} are the two most popular methods in reconstructing the graph structure and the number of samples required is even near-optimal with respect to (w.r.t.) previously established information-theoretic lower-bound \cite{santhanam2012information}.

In this paper, we consider the well-known least absolute shrinkage
and selection operator (Lasso) \cite{tibshirani1996regression} for Ising model selection. At first sight, one might even doubt its suitability for this problem since apparently the Ising snapshots are binary data generated in a nonlinear manner while Lasso is (presumably) used for continuous data with linear regression. In fact, the idea of using linear regression for binary data is not as outrageous (or naive) as one might imagine \cite{brillinger1982generalized,dobriban2018high,erdogdu2019scalable}, and perhaps surprisingly, sometimes linear regression even outperforms logistic regression as demonstrated in \cite{gomila2021logistic}. Indeed, if our goal is to make predictions of new outcomes, say binary classification, then linear regression might not be a good choice since it is easily prone to out-of-bound forecasts\footnote{In fact, even for classification, linear regression is widely used, e.g., ridge classification \cite{dobriban2018high}, which can be significantly faster than logistic regression with a high number of classes \cite{RidgeClassifier}.}. However, when it comes to other goals such as estimating variables or causal effects \cite{gomila2021logistic}, the answer becomes highly nontrivial. For Ising model selection, the goal is not about making predictions of new binary outcomes, but rather inferring the graph structure and thus deciphering the underlying conditional independence between different variables. Hence, given the popularity of Lasso, it is of both practical and theoretical significance to study the (mis-specified) Lasso's \textit{model selection consistency} for the nonlinear Ising models, i.e., under what conditions Lasso can (or cannot) successfully recover the true structure of Ising model. While several early studies \cite{bento2009graphical,lokhov2018optimal, meng2020structure,meng2021ising} have implied Lasso's potential consistency for Ising model selection, a rigorous theoretical analysis has still largely remained  unresolved.

\subsection{Our Contributions}
We theoretically analyze the model selection consistency of Lasso, \textit{both with and without post-thresholding}, for Ising models in the  high-dimensional
($n\ll p$) regime, where the number of vertices $p=p\left(n\right)$ may also scale as a function of the sample
size $n$. The paramagnetic phase of Ising models is considered where the coupling strength is relatively small so that the expectation of the magnetization  $m:=\frac{1}{p}\sum_{i=1}^p x_i$  is zero \cite{nishimori2001statistical, mezard2009information}. 
Our main contributions are summarized as follows. 

(a) For random regular (RR) graphs with regular  node degree $d$ and  uniform active couplings $\theta^{*}_{r,t} = \theta_0,\forall (r,t)\in E$, in  the paramagnetic phase, i.e., $(d-1)\tanh{\theta_0}<1$, we prove that Lasso without post-thresholding is  model selection consistent for Ising models, and remarkably, the required sample complexity has the same scaling order as that of $\ell_1$-LogR.  (Theorem \ref{theorem-RR-Lasso-no-post})

(b) For general tree-like graphs, under mild assumptions of the \textit{dependency condition} and \textit{incoherence condition}, it is proved that Lasso without post-thresholding is still model selection consistent for Ising models with the same order of  sample complexity as that of $\ell_1$-LogR. (Theorem \ref{theorem-main-tree-graph-noPost}) 

(c) For general tree-like graphs, we not only obtain an upper  bound of the reconstructed square error of Lasso, but also prove that, with some proper post-thresholding, Lasso is model selection consistent with the same order of  sample complexity as that of $\ell_1$-LogR and RISE without any further assumptions on the \textit{dependency} and \textit{incoherence conditions}. (Theorems \ref{theorem-square-error} and  \ref{theorem-Post-Thresholding})  

\textit{Remark} 1: It is worth strengthening that in this paper we  focus on Lasso \textit{both with and without post-theresholding}.

\textit{Remark} 2: Given the wide popularity and efficiency of Lasso, our analysis not only provides a theoretical backing for its practical use, but also deepens our understanding of learning Ising models using Lasso. Previously, it has long been believed that the success of Lasso for Ising model selection (approximately) happens only when $\theta^{*}_{r,t}\to 0,\forall (r,t)\in E$ so that the square loss of Lasso is similar to the logistic loss of $\ell_1$-LogR \cite{lokhov2018optimal}. However, we identify and prove that Lasso actually behaves similarly as $\ell_1$-LogR and RISE in the whole paramagnetic phase (as opposed to the limit regime $\theta^{*}_{r,t}\to 0,\forall (r,t)\in E$). We hope that our study could inspire further research on alternative simple and efficient methods for Ising model selection.   

\subsection{Related  Works}
In \cite{bento2009graphical}, the authors pointed out a potential relevance of the incoherence condition of Lasso \cite{zhao2006model} to $\ell_1$-LogR by expanding the logistic loss around the true interactions $\theta^{*}$. However, on the one hand, it is restricted to the case when the $\ell_1$ regularization parameter approaches zero. On the other hand, the resultant quadratic loss is actually different from that of Lasso. Later, \cite{lokhov2018optimal} observed that at high temperatures when the magnitude of interactions approaches zero, i.e., $\theta^{*}_{r,t}\to 0,\forall (r,t)\in E$, both the logistic and interaction screening objective (ISO) losses can be approximated as a square loss using a second-order Taylor expansion around zero (as opposed to $\theta^{*}$ in \cite{zhao2006model}). However, their results are severely limited to the regime $\theta^{*}_{r,t}\to 0,\forall (r,t)\in E$. In other words, \cite{lokhov2018optimal} attributed the potential success of Lasso to its similarity with $\ell_{1}$-LogR/RISE in the regime $\theta^{*}_{r,t}\to 0,\forall (r,t)\in E$. Moreover, without considering the $\ell_1$ regularization term, \cite{lokhov2018optimal} only compared the analytical solution with that of the naive mean-field method \cite{tanaka1998mean,kappen1998efficient,ricci2012bethe}. A rigorous theoretical analysis of the consistency of Lasso  for Ising model selection is still lacking.


To the best of our knowledge, the first explicit analysis of Lasso for Ising model selection is given in  \cite{meng2021ising} using statistical physics methods, building on previous studies \cite{bachschmid2017statistical,Abbara2019c, meng2020structure}.  In particular, \cite{meng2021ising} demonstrated that Lasso has the same order of sample complexity as $\ell_{1}$-LogR for random regular (RR) graphs in the paramagnetic phase \cite{mezard2009information}. Furthermore, \cite{meng2021ising} provided an accurate estimate of the typical sample complexity as well as a precise prediction of the non-asymptotic learning performance. However, there are several limitations in \cite{meng2021ising}. First, since the replica method \cite{opper2001advanced, nishimori2001statistical, mezard2009information} they use is a non-rigorous method from statistical mechanics, a rigorous mathematical proof has remained lacking. Second, the results in \cite{meng2021ising} are  restricted to the special class of RR graphs. In addition, since their analysis relies on the {\em self averaging property} \cite{nishimori2001statistical,mezard2009information}, the results in \cite{meng2021ising} are meaningful in terms of the ``typical case'' \cite{engel2001statistical} rather than the worst case. 
Moreover, \cite{meng2021ising} did not analyze the case of Lasso with post-thresholding.  

Regarding the study of Lasso for nonlinear (not necessarily binary) targets, the past few years have seen an active line of research in the field of signal processing with a special focus on the single-index model \cite{brillinger1982generalized,plan2016generalized,thrampoulidis2015lasso,zhang2016consistency,genzel2016high}. These studies are related to ours but with several important differences. First, in our study, the covariates are generated from an Ising model rather than a Gaussian distribution. Second, we focus on model selection consistency of Lasso while most previous studies considered estimation consistency except \cite{zhang2016consistency}. However,  \cite{zhang2016consistency} only considered the classical asymptotic regime while we are interested in the high-dimensional setting where  $n \ll p$. Another closely related work is \cite{erdogdu2019scalable}, which studied the relationship between the true minimizer of the population risk of a generalized linear model and the ordinary least square coefficient. Nevertheless, they only focused on the classic $n \gg p$ case. Moreover,  even in the classic case,  \cite{erdogdu2019scalable} did not provide a rigorous analysis of the model selection consistency of Lasso with the empirical risk. 

\subsection{Notations}
For each vertex $r\in V$, the neighborhood set is denoted as $\mathcal{N}\left(r\right)\coloneqq\left\{ t\in V|\left(r,t\right)\in E\right\} $, the signed neighborhood set is defined as
$
\mathcal{N}_{\pm}\left(r\right)\coloneqq\left\{ \textrm{sign}\left(\theta_{rt}^{*}\right)t|t\in\mathcal{N}\left(r\right)\right\}$, and the corresponding node degree is denoted as $d_{r}\coloneqq\left|\mathcal{N}\left(r\right)\right|$.
The maximum node degree of the whole graph $G$ is denoted as $d\coloneqq\underset{r\in V}{\max}\;d_{r}$. We use $\mathcal{G}_{p,d}$
to denote the ensemble of graphs $G$ with $p$ vertices and maximum
(not necessarily bounded) node degree $d\geq3$. The minimum and maximum magnitudes of the interactions $\theta_{rt}^{*}$
for $\left(r,t\right)\in E$ are respectively denoted as
\begin{align}
\theta_{\min}^{*}\coloneqq  \underset{\left(r,t\right)\in E}{\min}\left|\theta_{rt}^{*}\right|,\;\;
\theta_{\max}^{*}\coloneqq  \underset{\left(r,t\right)\in E}{\max}\left|\theta_{rt}^{*}\right|.\label{eq:interactions-max}
\end{align}
$\mathbb{E}_{\theta^{*}}\left\{\cdot\right\}$ denotes expectation w.r.t. the joint distribution $\mathbb{P}_{\theta^{*}}\left(x\right)$  (\ref{eq:Ising_distribution}). $\interleave A \interleave_{\infty} = \max_{j}\sum_{k} \left|A_{jk}\right|$ is the $\ell_{\infty}$ matrix norm of a matrix $A$. $\Lambda_{\min}\left(A\right)$ and $\Lambda_{\max}\left(A\right)$ denote the minimum and maximum eigenvalue of $A$, respectively. 

\section{\label{sec:Background}Problem Setup}
The problem of Ising model selection can be generally described as follows:  given a collection of $n$ i.i.d. samples $\mathfrak{X}_n\coloneqq\left\{ x^{(1)},\ldots,x^{(n)}\right\}$ from an Ising model defined on a graph $G=\left(V,E\right)$, the goal is to reconstruct the graph structure of $G$. In this paper we focus on Ising models defined on general locally tree-like graphs, i.e., the local neighborhood of a uniformly random vertex of the graph converges in distribution to a random rooted tree \cite{dembo2010ising}. In particular, we also pay a special attention to the popular random regular (RR) graphs, one typical class of locally tree-like graphs with regular node degree $d_r = d$ and  uniform couplings  $\theta^{*}_{r,t} = \theta_0,\forall (r,t)\in E$.

As in \cite{ravikumar2010high}, we
consider the slightly stronger criterion of \textit{signed edge} recovery, and investigate the sufficient conditions on the \textit{sparsistency property}. 
\begin{definition}
(\textit{signed edge})
The signed edge set $E^{*}$ of one Ising model with interactions $\theta^{*}$ is defined as $E^{*}:=\left\{\textrm{sign}\left(\theta_{rt}^{*}\right)\right\}$
where $\textrm{sign}\left(\cdot\right)$ is an element-wise operation
that maps every positive entry to 1, negative entry to -1, and zero entry
to zero. 
\end{definition}

\begin{definition}
(\textit{sparsistency} property) Suppose that $\hat{E}_{n}$ is an estimator of the signed edge  $E^{*}$ given $\mathfrak{X}_n$, then it is called (signed) \textit{model selection consistent} in the sense that
\begin{equation}
\mathbb{P}\left(\hat{E}_{n}=E^{*}\right)\rightarrow1\textrm{ as }n\rightarrow+\infty,\label{eq:model-consistent}
\end{equation}
which is known as the \textit{sparsistency} property \cite{ravikumar2010high}.
\end{definition}
Our goal is to investigate the \textit{sparsistency} property of Lasso \cite{tibshirani1996regression} for high-dimensional Ising models on locally tree-like graphs. 
Since recovering the edge set $E^{*}$ of any graph $G=\left(V,E\right)$ is equivalent to reconstructing the associated signed neighborhood set $
\mathcal{N}_{\pm}\left(r\right)\coloneqq\left\{ \textrm{sign}\left(\theta_{rt}^{*}\right)t|t\in\mathcal{N}\left(r\right)\right\}$ for each vertex $r\in V$ \cite{ravikumar2010high},  one can equivalently investigate the scaling
condition on $\left(n,p,d\right)$ which ensures that the estimated
signed neighborhood $\mathcal{\hat{N}}_{\pm}\left(r\right)$ agrees
with the true neighborhood, i.e., $\big\{ \mathcal{\hat{N}}_{\pm}\left(r\right)=\mathcal{N}_{\pm}\left(r\right),\forall r\in V\big\} $,
with high probability.

Specifically, the estimate of the sub-vector $\theta_{\setminus r}^{*}\coloneqq\left\{ \theta_{rt}^{*}|t\in V\setminus r\right\} \in\mathbb{R}^{p-1},\;\forall r\in V$ is obtained via Lasso as follows
\begin{align}
\hat{\theta}_{\setminus r} & =\underset{\theta_{\setminus r}}{\arg\min}\left\{ \ell\left(\theta_{\setminus r};\mathfrak{X}_n\right)+\lambda_{\left(n,p,d\right)}\left\Vert \theta_{\setminus r}\right\Vert _{1}\right\}, \label{eq:lasso-estimator}
\end{align}
where $\ell\left(\theta_{\setminus r};\mathfrak{X}_n\right)$
denotes the square loss function
\begin{equation}
\ell\left(\theta_{\setminus r};\mathfrak{X}_n\right)\coloneqq\frac{1}{2n}\sum_{i=1}^{n}\big(x_{r}^{(i)}-\sum_{u\in V\setminus r}\theta_{ru}x_{u}^{(i)} \big)^{2},\label{eq:loss-def}
\end{equation}
and $\lambda_{\left(n,p,d\right)}>0$ is the regularization parameter. For simplicity, instead of $\lambda_{\left(n,p,d\right)}$, $\lambda_{n}$ will
be used hereafter.

Subsequently, one can obtain an estimate $\mathcal{\hat{N}}_{\pm}\left(r\right)$ of  $\mathcal{N}_{\pm}\left(r\right)$ from the Lasso results $\hat{\theta}_{\setminus r}$ in (\ref{eq:lasso-estimator}). Here we focus on two different settings:  \textit{without post-theresholding} and \textit{with post-thresholding}.  Without post-theresholding, one can simply estimate $\mathcal{\hat{N}}_{\pm}\left(r\right)$ using the sign information as \cite{ravikumar2010high}
\begin{equation}
\mathcal{\hat{N}}_{\pm}\left(r\right)\coloneqq\left\{ \textrm{sign}\left(\hat{\theta}_{rt}\right)t|t\in V\setminus r,\hat{\theta}_{rt}\neq0\right\}.\label{eq:Neighbor-est}
\end{equation}
Alternatively, one introduces a threshold $\xi>0$ and then perform post-thresholding on $\hat{\theta}_{\setminus r}$ \cite{ekeberg2013improved,decelle2014pseudolikelihood,lokhov2018optimal}, leading to
\begin{equation}
\mathcal{\hat{N}}_{\pm}\left(r\right)\coloneqq\left\{ \textrm{sign}\left(\hat{\theta}_{rt}\right)\mathtt{1}\left(\textrm{\ensuremath{\left|\hat{\theta}_{rt}\right|}}>\xi\right)t|t\in V\setminus r,\hat{\theta}_{rt}\neq0\right\}, \label{eq:Neighbor-est-thresh}
\end{equation}
where $\mathtt{1}\left(\cdot\right)$ is an indicator function that
equals to 1 if the event is true and 0 otherwise. 

\section{\label{sec: mainresults}Main results }
\subsection{Preliminary Results}
Before stating the main results, we first present two different results of Lasso compared with $\ell_1$-LogR regarding the expected first and second derivative of the loss function, i.e., $\mathbb{E}_{\theta^{*}}\{\nabla\ell\left(\theta_{\setminus r};\mathfrak{X}_{1}^{n}\right)\}$ and  $\mathbb{E}_{\theta^{*}}\{\nabla^2\ell\left(\theta_{\setminus r};\mathfrak{X}_{1}^{n}\right)\}$. 
\begin{lem}
\label{lemma-recaled-solution} For general tree-like graphs in the paramagnetic
phase, the solution to $\mathbb{E}_{\theta^{*}}\{\nabla\ell\left(\theta_{\setminus r};\mathfrak{X}_{1}^{n}\right)\}=0$,
denoted as $\tilde{\theta}_{\setminus r}^{*}=\left\{ \tilde{\theta}_{rt}^{*}\right\} _{t\in V\setminus r}\in\mathbb{R}^{p-1}$,
can be obtained as
\begin{equation}
\tilde{\theta}_{rt}^{*}=\begin{cases}
\frac{\tanh\left(\theta_{rt}^{*}\right)/\left({1-\tanh^{2}\left(\theta_{rt}^{*}\right)}\right)}{1-d_{r}+\sum_{u\in\mathcal{N}\left(r\right)}\frac{1}{1-\tanh^{2}\left(\theta_{ru}^{*}\right)}} & \textrm{if}\left(r,t\right)\in E\\
0 & \textrm{otherwise.}
\end{cases}\label{eq:rescaled-result}
\end{equation}
where $d_{r}$ is the node degree of $r$. In particular, for $RR$ graph with uniform coupling strength $\theta_{rt}^{*}=\theta_{0},\forall\left(r,t\right)\in E$
and constant  node degree $d_{r}=d$, there is
\begin{equation}
\tilde{\theta}_{rt}^{*}=\begin{cases}
\frac{\tanh\left(\theta_{0}\right)}{1+\left(d-1\right)\tanh^{2}\left(\theta_{0}\right)} & \textrm{if}\left(r,t\right)\in E\\
0 & \textrm{otherwise.}
\end{cases}\label{eq:rescaled-result-RR}
\end{equation}
\end{lem}
\begin{proof}
See Appendix \ref{appendix:lemma-recaled-solution}. 
\end{proof}
Lemma \ref{lemma-recaled-solution} indicates that, the solution $\tilde{\theta}_{\setminus r}^{*}$ is a rescaled value of the true parameter $\theta_{\setminus r}^{*}$ and thus shares the same sign structure, i.e., $\textrm{sign}\left(\tilde{\theta}_{\setminus r}^{*}\right)=\textrm{sign}\left(\theta_{\setminus r}^{*}\right)$.  The minimum magnitude of $\tilde{\theta}_{rt}^{*}$ for $\left(r,t\right)\in E$ in (\ref{eq:rescaled-result}) is denoted as
\begin{equation}
\tilde{\theta}_{\min}^{*}\coloneqq\underset{\left(r,t\right)\in E}{\min}\tilde{\theta}_{rt}^{*}.  \label{eq:min-rescale-def}
\end{equation}
For the second derivative or Hessian matrix, in the case of Lasso,  it corresponds exactly to the covariance matrix, i.e., 
\begin{equation}
Q_{r}^{*}\coloneqq\mathbb{E}_{\theta^{*}}\{\nabla^2\ell\left(\theta_{\setminus r};\mathfrak{X}_{1}^{n}\right)\} =\mathbb{E}_{\theta^{*}}\{ X_{\setminus r}X_{\setminus r}^{T}\}, \forall r\in V .\label{eq:Hessian-Qr}
\end{equation}
As opposed to \cite{ravikumar2010high}, the additional variance function term of  $\ell_1$-LogR (eq. (12) in  \cite{ravikumar2010high}, denoted as $\eta(X;\theta^*)$) does not exist in $Q_{r}^{*}$ ( \ref{eq:Hessian-Qr}), which makes Lasso different from $\ell_1$-LogR, including its behavior and the corresponding proof. For notational simplicity, $Q_{r}^{*}$ will be written as $Q^{*}$
hereafter. Denote $S\coloneqq\left\{ \left(r,t\right)\mid t\in\mathcal{N}\left(r\right)\right\} $
as the subset of indices associated with edges of $r$ and $S^{c}$
as its complement. The $d_r\times d_r$ sub-matrix of $Q^{*}$ indexed
by $S$ is denoted as $Q_{SS}^{*}$. Other sub-matrices like  $Q^*_{S^c S}$ are defined in the same way.
\subsection{Lasso without Post-thresholding}
For Lasso without post-thresholding, i.e., the signed edge set $\mathcal{\hat{N}}_{\pm}\left(r\right),\forall r\in V$ is obtained as (\ref{eq:Neighbor-est}), we have  
\begin{theorem}
\label{theorem-RR-Lasso-no-post} (RR graphs) Consider a collection of $n$ i.i.d. samples $\mathfrak{X}_n\coloneqq\left\{ x^{(1)},\ldots,x^{(n)}\right\}$ drawn from an Ising model on a RR graph $G=\left(V,E\right)\in\mathcal{G}_{p,d}$ with regular  node degree $d$ and uniform couplings $\theta^{*}_{r,t} = \theta_0,\forall (r,t)\in E$. Suppose that the Ising model is in the paramagnetic phase, i.e., $(d-1)\tanh{\theta_0}<1$, then there exist constants $L,c$ independent of $\left(n,p,d\right)$, so that the Lasso estimator (\ref{eq:lasso-estimator}) with the regularization parameter $\lambda_n\leq \frac{\tanh\left(\theta_{0}\right)(1-\tanh^2{(\theta_0)})}{6\sqrt{d}\left(1+\left(d-1\right)\tanh^{2}\left(\theta_{0}\right)\right)}$ reconstructs the signed edge set by (\ref{eq:Neighbor-est}) perfectly with probability at least 
\begin{align}
    \mathbb{P}\left(\hat{E}_{n}=E^{*}\right) \geq 1-2\exp\left(-c\lambda_n^2n\right)
\end{align}
as long as $n \geq \max\left\{Ld^3, \frac{64(1+\tanh{(\theta_0)})^2}{(1-\tanh{(\theta_0)})^2\lambda_n^2}\right\}\log p$. 
\end{theorem}

\textit{Remark} 3: Theorem \ref{theorem-RR-Lasso-no-post} indicates that the probability that the Lasso estimator (\ref{eq:lasso-estimator}) successfully recovers the true signed edge set decays exponentially as a function of $\lambda_n^2 n$, which is the same as $\ell_1$-LogR \cite{ravikumar2010high}.  If $\lambda_n$ is chosen such that $\lambda_n^2 n\rightarrow \infty$ as $n\to \infty$, Lasso is model selection consistent, i.e., $ \mathbb{P}\left(\hat{E}_{n}=E^{*}\right)\to 1$ as $n\to\infty$. 
In the high-dimensional case, one reasonable choice of $\lambda_n$ that satisfies both Theorem \ref{theorem-RR-Lasso-no-post} and $\lambda_n^2 n\rightarrow \infty$ is $\lambda_n=\kappa\sqrt{\frac{\log p}{n}}$, where $\kappa \geq \frac{8(1+\tanh{\theta_0})}{1-\tanh{\theta_0}}$. In this case, i.e., $\lambda_n=\kappa\sqrt{\frac{\log p}{n}}$, from Theorem \ref{theorem-RR-Lasso-no-post}, it is obtained that the number of samples required for model selection consistency needs to satisfy $n \geq \max\left\{Ld^3, \frac{36\kappa^2(1+(d-1)\tanh{(\theta_0)})^2}{\tanh^2{(\theta_0)})(1-\tanh^2{(\theta_0)})^2}d\right\}\log p$.  {In practical applications, one might not be able to obtain the ideal  $\kappa$ when $\theta_0$ is unknown. However, a larger   $\kappa$ can be chosen, which is possible either by using a prior knowledge of the range of $\theta_0$ for a specific problem, or by setting $\kappa$ as large as possible. A disadvantage of larger $\kappa$ is that, the number of samples $n$  becomes larger since $n \geq \max\left\{Ld^3, \frac{36\kappa^2(1+(d-1)\tanh{(\theta_0)})^2}{\tanh^2{(\theta_0)})(1-\tanh^2{(\theta_0)})^2}d\right\}\log p$. However, this is an inevitable price we have to pay due to  the lack of knowledge of the models. }

{Note that while the uniform coupling of RR graphs in Theorem \ref{theorem-RR-Lasso-no-post} is a limitation,  Theorem 1 holds  without additional assumptions as $\ell_1$-LogR \cite{ravikumar2010high}.}  For general locally tree-like graphs, under additional mild assumptions similar to $\ell_1$-LogR \cite{ravikumar2010high}, namely the \textsl{dependency condition} and \textsl{incoherence condition},  we can still obtain similar results as RR graphs in Theorem \ref{theorem-RR-Lasso-no-post}. 

\textbf{Condition 1} (C1): \textit{dependency condition}. The sub-matrix $Q_{SS}^{*}$ has bounded eigenvalue, i.e., there exists a constant $C_{\min}>0$
such that
\begin{align}
\Lambda_{\min}\left(Q_{SS}^{*}\right) & \geq C_{\min}. \label{eq:eigen-upper-lower}
\end{align}
\textbf{Condition 2} (C2): \textit{incoherence condition}. 
There exists an $\alpha\in(0,1]$ such that
\begin{equation}
\interleave Q_{S^{c}S}^{*}\left(Q_{SS}^{*}\right)^{-1}\interleave_{\infty}\leq1-\alpha. \label{eq:incoherence-cond}
\end{equation}

\begin{theorem}
\label{theorem-main-tree-graph-noPost} (tree-like graphs) Consider general tree-like graphs $G=\left(V,E\right)\in\mathcal{G}_{p,d}$ in the paramagnetic phase. Suppose that conditions (C1) and (C2) are satisfied by the population covariance matrix $Q^*$. If the regularization parameter $\lambda_{n}$ is selected to satisfy $\lambda_{n}\geq\frac{4\sqrt{c+1}\left(2-\alpha\right)}{\alpha}\sqrt{\frac{\log p}{n}}$
for some constant $c>0$, then there exists a constant $L$ independent
of $\left(n,p,d\right)$ such that if 
\begin{equation}
n\geq Ld^{3}\log p,\label{eq:n-condition-population}
\end{equation} then with probability at least $1-2\exp\left(-c\log p\right)\rightarrow1$ as $p\rightarrow \infty$, the following properties hold:

(a) For each node $r\in V$, the Lasso estimator (\ref{eq:lasso-estimator}) has a unique solution, and thus uniquely specifies a signed neighborhood  $\mathcal{\hat{N}}_{\pm}\left(r\right)$ with (\ref{eq:Neighbor-est}).

(b) For each node $r\in V$, the estimated signed neighborhood vector $\mathcal{\hat{N}}_{\pm}\left(r\right)$ with (\ref{eq:Neighbor-est}) correctly excludes all edges not in the true neighborhood. Moreover, it correctly includes all edges if the minimum magnitude of the rescaled parameter satisfies  $\tilde{\theta}_{\min}^{*}\geq\frac{6\lambda_{n}\sqrt{d}}{C_{\min}}$.
\end{theorem}

\textit{Remark} 4: Theorem \ref{theorem-main-tree-graph-noPost} indicates that the probability that Lasso recovers the true signed edge set $\mathbb{P}\left(\hat{E}_{n}=E^{*}\right)\to 1$ exponentially as a function of $\log p$. Hence, for tree-like Ising models in the paramagnetic phase, under conditions (C1) and (C2), in the high-dimensional setting (for $p\to \infty$), Lasso is model selection consistent with $n=\Omega{(d^3\log{p})}$ samples, which is the same as $\ell_1$-LogR \cite{ravikumar2010high}.  

In contrast to  RR graphs in Theorem \ref{theorem-RR-Lasso-no-post}, for general tree-like graphs, two additional assumptions (C1) and (C2) are imposed for the success of Lasso without post-thresholding. However, it is worth noting that  $\ell_1$-LogR without post-thresholding also suffers from the same limitation as shown in \cite{ravikumar2010high}, which is due to the fundamental difficulty in verifying (C1) and (C2) for general graphs.

\subsection{Lasso with Post-thresholding}
For Lasso with post-thresholding, i.e., the signed neighborhood set $\mathcal{\hat{N}}_{\pm}\left(r\right),\forall r\in V$ is obtained as (\ref{eq:Neighbor-est-thresh}), we obtain the following results. 
\begin{theorem}
\label{theorem-square-error} (Square error, tree-like graphs) Consider an
Ising model defined on tree-like graphs $G=\left(V,E\right)\in\mathcal{G}_{p,d}$. $\forall r\in V$
and for any $\varepsilon_{1}>0$, in the paramagnetic phase, the square error of the Lasso estimator
(\ref{eq:lasso-estimator}) with regularization parameter $\lambda_{n}=4\sqrt{\frac{\log\frac{3p}{\varepsilon_{1}}}{n}}$
is bounded with probability at least $1-\varepsilon_{1}$ by
\begin{equation}
\left\Vert \hat{\theta}_{\setminus r}-\tilde{\theta}_{\setminus r}^{*}\right\Vert _{2}\leq2^{6}\sqrt{d}\left(d+1\right)e^{2\theta_{\max}^{*}d}\sqrt{\frac{\log\frac{3p}{\varepsilon_{1}}}{n}}\label{eq:square-error-bound-1}
\end{equation}
when $n\geq2^{14}d^{2}\left(d+1\right)^{2}e^{4\theta_{\max}^{*}d}\log\frac{3p^{2}}{\varepsilon_{1}}$. \label{eq:square-error-bound}
\end{theorem}

\begin{theorem}
\label{theorem-Post-Thresholding} (Structure learning, tree-like graphs) Consider
an Ising model defined on tree-like graphs $G=\left(V,E\right)\in\mathcal{G}_{p,d}$. In the paramagnetic phase, for any $\varepsilon_2>0$,
the Lasso estimator (\ref{eq:lasso-estimator}) with regularization
parameter $\lambda_{n}=4\sqrt{\frac{\log\frac{3p^{2}}{\varepsilon_2}}{n}}$
reconstructs the sign edge set by (\ref{eq:Neighbor-est-thresh}) perfectly with probability
\begin{equation}
\mathbb{P}\left(\hat{E}=E^{*}\right)\geq1-\varepsilon_2,\label{eq:square-error-bound-1-1}
\end{equation}
as long as 
\begin{equation}
n\geq\max\left\{ d,\left(\tilde{\theta}_{\min}^{*}\right)^{-2}\right\} 2^{14}d\left(d+1\right)^{2}e^{4\theta_{\max}^{*}d}\log\frac{3p^{3}}{\varepsilon_2}. \label{eq:Lasso-post-sample-bound}
\end{equation}
\end{theorem}

\textit{Remark} 5: Results in Theorems \ref{theorem-square-error} and  \ref{theorem-Post-Thresholding} hold for general tree-like graphs without any further assumptions of (C1) and (C2). In particular, Theorem \ref{theorem-Post-Thresholding} indicates that Lasso with post-thresholding is model selection consistent under similar conditions as the RISE \cite{vuffray2016interaction} and $\ell_1$-LogR with post-thresholding \cite{lokhov2018optimal}. Note that similarly as RISE and $\ell_1$-LogR \cite{vuffray2016interaction,lokhov2018optimal}, the obtained bound in (\ref{eq:Lasso-post-sample-bound}) is a rather loose bound, especially in the paramagnetic phase, e.g., while it suggests an exponential growth w.r.t. $\theta_{\max}^{*}$, it is in fact not the case in the paramagnetic phase (see Figure 4 in \cite{lokhov2018optimal}). 

{\textit{Remark} 6: In Theorems \ref{theorem-square-error} and  \ref{theorem-Post-Thresholding}, the regularization parameter $\lambda_{n}$ is chosen  as a function of the controlled probability values $\epsilon_1$ or $\epsilon_2$. In practical applications when no prior knowledge is available, similarly as \cite{vuffray2016interaction,lokhov2018optimal}, we can simply choose $\lambda_{n}$ as
$\lambda_{n}=4\sqrt{\frac{\log\frac{3p}{\epsilon_{1}}}{n}}$ or $\lambda_{n}=4\sqrt{\frac{\log\frac{3p^{2}}{\epsilon_2}}{n}}$, respectively. Moreover, as described in \cite{lokhov2018optimal}, given a sufficient number of samples, other techniques such as consistency cross-validation can be used for selecting the optimal value of $\lambda_{n}$ on a case-by-case basis. For more details, please refer to the supplementary material of \cite{lokhov2018optimal}. }

\section{Proof of the main results}
Here we provide a sketch of the proofs for the main results. For details, please refer to Appendices \ref{appendix-main-proof-no-thd} and \ref{appendix-main-proof-with-thd}. 
\subsection{Sketch of the proof for Theorems \ref{theorem-RR-Lasso-no-post} and \ref{theorem-main-tree-graph-noPost}}
For the proof of Lasso without post-thresholding, we use the primal-dual witness proof framework \cite{ravikumar2010high}, which was originally proposed in \cite{Wainwright2009sharp}. The main idea of the primal-dual witness
method is to explicitly construct an optimal primal-dual pair which
satisfies the sub-gradient optimality conditions associated with the Lasso estimator (\ref{eq:lasso-estimator}). Subsequently, it is proved
that under the stated assumptions on $\left(n,p,d\right)$, the optimal
primal-dual pair can be constructed such that they act as a witness,
i.e., a certificate that guarantees that the neighborhood-based Lasso estimator  (\ref{eq:lasso-estimator}) together with (\ref{eq:Neighbor-est}) correctly
recovers the signed edge set of the graph $G\in\mathcal{G}_{p,d}$.


Generally speaking, the proof of Theorems \ref{theorem-RR-Lasso-no-post} and \ref{theorem-main-tree-graph-noPost}  consists of two stages. At the first stage, we  consider a ``fixed design'' case assuming that the sample Hessian $Q^{n}\coloneqq\frac{1}{n}\sum_{i=1}^{n}x_{\setminus r}^{\left(i\right)}\left(x_{\setminus r}^{\left(i\right)}\right)^{T}$, satisfies both conditions (C1) and (C2). Afterwards, at the second stage, using some large-deviation analysis we provide guarantees under which both conditions (C1) and (C2) hold for the
sample Hessian $Q^{n}$ with high probability. Finally, we obtain Theorems \ref{theorem-RR-Lasso-no-post} and \ref{theorem-main-tree-graph-noPost} combining results of the two stages. Notably, for RR graphs, there is one remarkable  property, as shown in Lemma \ref{lem:RR-graph-(C1)(C2)}: 
\begin{lem}
\label{lem:RR-graph-(C1)(C2)}
For Ising models defined on RR graphs $G=\left(V,E\right)\in\mathcal{G}_{p,d}$ with regular  node degree $d$ and uniform couplings $\theta^{*}_{r,t} = \theta_0,\forall (r,t)\in E$. 
In the paramagnetic phase, both conditions (C1) and (C2) hold for $Q^{*}$, where $C_{\min} = 1 - \tanh^{2}\theta_{0}$ and $\alpha = 1 - \tanh\theta_{0}$. 
\end{lem}
\begin{proof}
See Appendix \ref{appendix:RR-C1C2-proof}. 
\end{proof}
As a result, in Theorem  \ref{theorem-RR-Lasso-no-post}, there is no need for assumptions (C1) and (C2) in the case of RR graphs.

The important results at the first stage are shown in Proposition \ref{theorem-main-fixed} and Proposition \ref{theorem-main-fixed-tree}, which correspond to the RR graphs and general tree-like graphs, respectively. 
\begin{prop} 
\label{theorem-main-fixed} (fixed design, RR graphs) Consider an Ising model on a RR graph $G=\left(V,E\right)\in\mathcal{G}_{p,d}$ with regular  node degree $d$ and uniform couplings $\theta^{*}_{r,t} = \theta_0,\forall (r,t)\in E$. Suppose that the Ising model is in the paramagnetic phase, and that the sample Hessian $Q^{n}$ satisfies (C1) and (C2). If the regularization parameter $\lambda_{n}$ 
satisfies $\lambda_n\geq \frac{8(2-\alpha)}{\alpha}\sqrt{\frac{\log p}{n}}$, then with probability at least $1-2\exp\left(-c\lambda_n^2n\right)\rightarrow1$, the following properties hold: 

(a) For each node $r\in V$, the Lasso estimator (\ref{eq:lasso-estimator}) has a unique solution, and thus uniquely specifies a signed neighborhood  $\mathcal{\hat{N}}_{\pm}\left(r\right)$.

(b) For each node $r\in V$, the estimated signed neighborhood vector $\mathcal{\hat{N}}_{\pm}\left(r\right)$ using the Lasso estimator (\ref{eq:lasso-estimator}) correctly excludes all edges not in the true neighborhood. Moreover, it correctly includes all edges if 
$\tilde{\theta}_{\min}^{*}\geq\frac{6\lambda_{n}\sqrt{d}}{C_{\min}}$. 
\end{prop}
\begin{proof}
See Appendix \ref{appendix:proof-proposition-RR}.
\end{proof}

\begin{prop}
\label{theorem-main-fixed-tree} (fixed design, tree-like graphs) Consider an Ising model defined on a tree-like graph $G=\left(V,E\right)\in\mathcal{G}_{p,d}$
with parameter vector $\theta^{*}$ and associated signed edge set $E^{*}$. Suppose that the Ising model is in the paramagnetic phase, and the sample Hessian $Q^{n}$  satisfies (C1) and (C2) and the regularization parameter $\lambda_{n}$ 
satisfies $\lambda_{n}\geq\frac{4\sqrt{c+1}\left(2-\alpha\right)}{\alpha}\sqrt{\frac{\log p}{n}}$
for some constant $c>0$. Under these conditions, if 
\begin{equation}
n\geq (c+1)d^{2}\log p,
\end{equation} then with probability at least $1-2\exp\left(-c\log p\right)\rightarrow1$ as $p\rightarrow \infty$, the following properties hold:

(a) For each node $r\in V$, the Lasso estimator (\ref{eq:lasso-estimator}) has a unique solution, and thus uniquely specifies a signed neighborhood  $\mathcal{\hat{N}}_{\pm}\left(r\right)$.

(b) For each node $r\in V$, the estimated signed neighborhood vector $\mathcal{\hat{N}}_{\pm}\left(r\right)$ correctly excludes all edges not in the true neighborhood. Moreover, it correctly includes all edges if 
$\tilde{\theta}_{\min}^{*}\geq\frac{6\lambda_{n}\sqrt{d}}{C_{\min}}$, where $\tilde{\theta}_{\min}^{*}$ is the minimum magnitude of the rescaled parameter $\tilde{\theta}^*$ defined in (\ref{eq:rescaled-result}). 
\end{prop}

\begin{proof}
See Appendix \ref{appendix:prop-tree-proof}.
\end{proof}

Note that in the above  two Propositions of the``fixed design'' case, in contrast to the ``fixed design'' results of $\ell_1$-LogR in \cite{ravikumar2010high}, there is no requirement of an additional scaling condition of $n\geq Ld^2\log p$. This is due to the fundamental difference between the square loss of Lasso and the logistic loss of $\ell_1$-LogR. Specifically for $\ell_1$-LogR, $n\geq Ld^2\log p$ is needed to ensure the $\ell_{2}$-consistency of the primal
sub-vector and to bound the remainder term, while it is not the case for Lasso with square loss, as shown in Lemma \ref{lem:L2-consistency}. However, this only holds under the assumption that the sample Hessian satisfies conditions (C1) and (C2). To ensure that these conditions are satisfied by the sample Hessian, an additional requirement of $n\geq Ld^3\log p$ is still needed, as shown in the final results in  Theorem  \ref{theorem-RR-Lasso-no-post} and Theorem \ref{theorem-main-tree-graph-noPost}. 

\textbf{Some key results}: The key results for the proofs of Lasso without post-thresholding are given as follows.
\begin{lem} 
\label{lem: Ws-Zs-results}
Denote $W^{n}=-\nabla\ell\left(\tilde{\theta}_{\setminus r}^{*};\mathfrak{X}_{1}^{n}\right)$. The $s$-th  element of $W^{n}$, denoted as   $W_{s}^{n}$, can be written as follows
\begin{align}
W_{s}^{n} & =\frac{1}{n}\sum_{i=1}^{n}Z_{s}^{\left(i\right)},\;\forall s\in V\setminus r,\label{eq:W-n-def}\\
Z_{s}^{\left(i\right)} &\coloneqq x_{s}^{(i)}(x_{r}^{(i)}-\sum_{t\in V\setminus r}\tilde{\theta}_{rt}^{*}x_{t}^{(i)}) \label{eq:Zs_def}. 
\end{align}
Then, $\mathbb{E}_{\theta^{*}}\left(Z_{s}^{\left(i\right)}\right)=0$,
$\mathtt{Var}\left(Z_{s}^{\left(i\right)}\right)\leq1$. Furthermore:

(a) For RR graphs, there is $\left|Z_{s}^{\left(i\right)}\right|\leq 2$;

(b) For general tree-like graphs, there is $\left|Z_{s}^{\left(i\right)}\right|\leq d$.
\end{lem}
\begin{proof}
See Appendix \ref{appendix:lemma-Z_s_bounded_zeromean_var}. 
\end{proof}
The behavior of $\left\Vert W^{n}\right\Vert _{\infty}$ is shown in Lemma \ref{lem:W_inf_norm}.
\begin{lem}
\label{lem:W_inf_norm}
\label{lem:W_inf}
Regarding $W^{n}=-\nabla\ell\left(\tilde{\theta}_{\setminus r}^{*};\mathfrak{X}_n\right)$ in Lemma \ref{lem: Ws-Zs-results}:
(a) For RR graphs, if $\lambda_n\geq \frac{8(2-\alpha)}{\alpha}\sqrt{\frac{\log p}{n}}$, then 
\begin{equation}
\mathbb{P}\Big(\frac{2-\alpha}{\lambda_{n}}\left\Vert W^{n}\right\Vert _{\infty}\geq\frac{\alpha}{2}\Big)  \leq 2\exp\Big(-\frac{\alpha^{2}\lambda_n^2 n}{32(2-\alpha)^2} + \log p\Big), \label{eq:W-inf-norm-RR}
\end{equation}
(b) For general tree-like graphs, if $n\geq\left(c+1\right)d^{2}\log p$ for some constant $c>0$ and
${\lambda_{n}\geq\frac{4\sqrt{c+1}\left(2-\alpha\right)}{\alpha}\sqrt{\frac{\log p}{n}}}$,
then 
\begin{equation}
\mathbb{P}\Big(\frac{2-\alpha}{\lambda_{n}}\left\Vert W^{n}\right\Vert _{\infty}\geq\frac{\alpha}{2}\Big)\leq2\exp\left(-c\log p\right).\label{eq:W-inf-norm-tree-like}
\end{equation}
\end{lem}
\begin{proof}
See Appendix \ref{appendix:lemma-W_inf_norm}.
\end{proof}
\begin{lem}
\label{lem:L2-consistency} If $\left\Vert W^{n}\right\Vert _{\infty}\leq\frac{\lambda_{n}}{2}$, then there is 
\begin{equation}
\left\Vert \hat{\theta}_{S}-\tilde{\theta}_{S}^{*}\right\Vert _{2}\leq\frac{3}{C_{\min}}\lambda_{n}\sqrt{d}.\label{eq:l2-consist}
\end{equation}
\end{lem}
\begin{proof}
See Appendix \ref{appendix:lem:L2-consistency}.
\end{proof}

\subsection{Sketch of the proof for Theorems \ref{theorem-square-error} and  \ref{theorem-Post-Thresholding}}
In proving Theorems \ref{theorem-square-error} and \ref{theorem-Post-Thresholding}, we resort to the restricted strong convexity framework in \cite{negahban2012unified}. 

First, consider the proof of Theorem \ref{theorem-square-error} which provides an estimation error bound (\ref{eq:square-error-bound-1}) of Lasso. Similarly as RISE and $\ell_1$-LogR \cite{vuffray2016interaction,negahban2012unified,lokhov2018optimal}, to obtain a handle on the (rescaled) square error of Lasso, two sufficient conditions (C3) and (C4) are enforced as follows: 

\textbf{Condition 3} (C3): The $\ell_{1}$ regularization parameter
$\lambda_{n}$ strongly enforces regularization if it is greater than
any partial derivatives of the loss function $\ell\left(\theta_{\setminus r};\mathfrak{X}_{1}^{n}\right)$
evaluated at $\tilde{\theta}_{\setminus r}^{*}$ defined in (\ref{eq:rescaled-result}), i.e., 
\begin{equation}
\lambda_{n}\geq2\left\Vert \nabla\ell\left(\tilde{\theta}_{\setminus r}^{*};\mathfrak{X}_{1}^{n}\right)\right\Vert _{\infty}.\label{eq:lambda-condion}
\end{equation}
Condition (C3) guarantees that if the vector of the rescaled couplings $\tilde{\theta}_{\setminus r}^{*}$
has at most $d$ non-zero elements, then the estimation difference
$\hat{\theta}_{\setminus r}-\tilde{\theta}_{\setminus r}^{*}$ lies
within the set

\begin{equation}
K\coloneqq\left\{ \triangle\in\mathbb{R}^{p-1}\mid\left\Vert \triangle\right\Vert _{1}\leq4\sqrt{d}\left\Vert \triangle\right\Vert _{2}\right\}. \label{eq:K-set-def}
\end{equation}

\textbf{Condition 4} (C4): The square loss of Lasso is restricted strongly
convex w.r.t. set $K$ (\ref{eq:K-set-def}) on a ball of radius $R$ centered at $\theta_{\setminus r}=\tilde{\theta}_{\setminus r}^{*}$
if for all $\triangle_{\theta_{\setminus r}}\in K$ such that $\left\Vert \triangle_{\theta_{\setminus r}}\right\Vert _{2}\leq R$,
there exists a constant $\kappa>0$ such that the remainder of the first-order Taylor expansion of the loss
function satisfies
\begin{equation}
\delta\ell\left(\triangle_{\theta_{\setminus r}},\tilde{\theta}_{\setminus r}^{*};\mathfrak{X}_{1}^{n}\right)\geq\kappa\left\Vert \triangle_{\theta_{\setminus r}}\right\Vert _{2}^{2}.\label{eq:lambda-condion-1}
\end{equation}
where $\triangle_{\theta_{\setminus r}}\in\mathbb{R}^{p-1}$ is an
arbitrary vector and the remainder can be calculated as 
\begin{equation}
\delta\ell\left(\triangle_{\theta_{\setminus r}},\tilde{\theta}_{\setminus r}^{*};\mathfrak{X}_{1}^{n}\right)=\frac{1}{2}\triangle_{\theta_{\setminus r}}^{T}Q^{n}\triangle_{\theta_{\setminus r}}.\label{eq:delta-loss-def}
\end{equation}
 
The key point is that, the estimation error $\big\Vert \hat{\theta}_{\setminus r}-\tilde{\theta}_{\setminus r}^{*}\big\Vert _{2}$ of Lasso can be controlled if conditions (C3) and (C4) are satisfied, as shown in  Proposition \ref{prop:square-error}: 
\begin{prop}
\label{prop:square-error} (Theorem 1, \cite{negahban2012unified}) If the Lasso estimator (\ref{eq:lasso-estimator})
satisfies both (C3) and (C4) with $R\geq3\sqrt{d}\frac{\lambda_{n}}{\kappa}$,
then the square error is bounded by 
\begin{equation}
\left\Vert \hat{\theta}_{\setminus r}-\tilde{\theta}_{\setminus r}^{*}\right\Vert _{2}\leq3\sqrt{d}\frac{\lambda_{n}}{\kappa}.\label{eq:square-error-bound}
\end{equation}
\end{prop}
As a result, the proof of Theorem \ref{theorem-square-error}
is done through Proposition \ref{prop:square-error}
by evaluating the two conditions (C3) and (C4). 

Regarding the proof of Theorem \ref{theorem-Post-Thresholding}, it is simply an application of Theorem \ref{theorem-square-error} by choosing a specific value of the estimation error. Specifically, with the definition of $\tilde{\theta}_{\min}^{*}$ in (\ref{eq:rescaled-result}) as the minimum rescaled coupling for a general graph, suppose that the estimated error $\big\Vert \hat{\theta}_{\setminus r}-\tilde{\theta}_{\setminus r}^{*}\big\Vert _{2}$ is controlled to be smaller than ${\tilde{\theta}_{\min}^{*}}/{2}$, then one can readily recover the structure of the neighborhood of node $r$ by setting the edges whose absolute estimated couplings are less than ${\tilde{\theta}_{\min}^{*}}/{2}$ to be absent \cite{lokhov2018optimal}. Subsequently, repeating this procedure over all the $p$  vertices, we are guaranteed through the union bound that exact reconstruction of the full edge set $E^*$ can be obtained with some predefined probability. 

\textbf{Some key results}: The key results for the proofs of Lasso with post-thresholding are given as follows. Specifically, Lemma \ref{lem:H-concentrate-result} and Lemma \ref{lem:H-bound-eign} are used to prove Lemma \ref{lem:strong-restricted}, which is then combined with Lemma \ref{lem:W_inf} to evaluate the conditions (C3) and (C4) via Proposition \ref{prop:square-error}, leading to the proof of Theorem  \ref{theorem-square-error}. 
\begin{lem}
\label{lem:W_inf}For any $\varepsilon_{3}>0$, if $n\geq d^{2}\log\frac{2p}{\varepsilon_{3}}$
, then probability at least $1-\varepsilon_{3}$
\begin{equation}
\left\Vert W^{n}\right\Vert _{\infty}\leq2\sqrt{\frac{\log\frac{2p}{\varepsilon_{3}}}{n}}.\label{eq:W-inf-norm-1}
\end{equation}
\end{lem}
\begin{proof}
See Appendix \ref{appendix:W_inf_convexity-proof}.
\end{proof}

The randomness of $\delta\ell\left(\triangle_{\theta_{\setminus r}},\tilde{\theta}_{\setminus r}^{*};\mathfrak{X}_{1}^{n}\right)$
can be controlled by $Q^{n}$, which concentrates towards its mean
independently of $\triangle_{\theta_{\setminus r}}$, as shown in
following lemma
\begin{lem}
\label{lem:H-concentrate-result}Let $\epsilon>0$, $\varepsilon_{4}>0$
and $n\geq\frac{2}{\epsilon^{2}}\log\frac{p^{2}}{\varepsilon_{4}}$,
then with probability greater than $1-\varepsilon_{4}$, we have for
all $s,t\in V\setminus r$
\[
\left|Q_{st}^{n}-Q^{*}_{st}\right|\leq\epsilon,
\]
where $Q_{st}^{n}=\frac{1}{n}\sum_{i=1}^{n}x_{t}^{(i)}x_{t}^{(i)}$ and  $Q^{*}_{st}=\mathbb{E}_{\theta^{*}}\left(x_{s}^{(i)}x_{t}^{(i)}\right),s,t\in V\setminus r$. 
\end{lem}
\begin{proof}
See Appendix \ref{lem:H-concentrate-result-proof}. 
\end{proof}

Lemma \ref{lem:H-bound-eign} states that the smallest eigenvalue of $Q^*$ is bounded below from zero independent of $p$.
\begin{lem}
\label{lem:H-bound-eign}(Lemma 7 in \cite{vuffray2016interaction})
For Ising model with graph $G\in\mathcal{G}_{p,d}$ with maximum coupling
strength $\theta_{\max}^{*}$. Then for all $\triangle_{\theta_{\setminus r}}\in\mathbb{R}^{p-1}$,
we have 
\[
\triangle_{\theta_{\setminus r}}^{T}Q^{*}\triangle_{\theta_{\setminus r}}\geq\frac{e^{-2\theta_{\max}^{*}d}}{d+1}\left\Vert \triangle_{\theta_{\setminus r}}\right\Vert _{2}^{2}.
\]
\end{lem}
Given the above results, the restricted strong convexity of the square
loss (\ref{eq:loss-def}) for Ising model problems is stated as follows.
\begin{lem}
\label{lem:strong-restricted}For Ising model with graph $G\in\mathcal{G}_{p,d}$
with maximum coupling strength $\theta_{\max}^{*}$, $\forall \varepsilon_{4}>0$, 
when $n>2^{11}d^{2}\left(d+1\right)^{2}e^{4\theta_{\max}^{*}d}\log\frac{p^{2}}{\varepsilon_{4}}$,
the square loss (\ref{eq:loss-def}) of Lasso satisfies, with probability at
least $1-\varepsilon_{4}$, the restricted strong convexity condition
\begin{equation}
\delta\ell\left(\triangle_{\theta_{\setminus r}},\tilde{\theta}_{\setminus r}^{*};\mathfrak{X}_{1}^{n}\right)\geq\frac{e^{-2\theta_{\max}^{*}d}}{4\left(d+1\right)}\left\Vert \triangle_{\theta_{\setminus r}}\right\Vert _{2}^{2}\label{eq:strong-restricted-cond}
\end{equation}
for all $\triangle_{\theta_{\setminus r}}\in\mathbb{R}^{p-1}$ such
that $\left\Vert \triangle_{\theta_{\setminus r}}\right\Vert _{1}\leq4\sqrt{d}\left\Vert \triangle_{\theta_{\setminus r}}\right\Vert _{2}$. 
\end{lem}
\begin{proof}
See Appendix \ref{appendix:strong-convexity-proof}. 
\end{proof}

\section{\label{sec:experimental results}Experimental Results}
In this section we conduct simulations to verify our theoretical findings that, simply speaking, Lasso performs similarly as $\ell_1$-LogR on typical tree-like graphs in the  paramagnetic phase. Two different structures of tree-like graphs are evaluated, namely RR graphs and star-shaped graphs. In addition, to have a look at the performance of Lasso for graphs with many loops,  we also evaluate the square lattice (grid) graphs with periodic boundary condition. It is worth noting that the RR and star-shaped graphs represent graphs with bounded node degree (the maximum node degree $d$ is a fixed constant) and unbounded node degree (the maximum node degree $d$ grows as the size of $p$), respectively. 

The experimental procedures are
as follows. First, a graph $G=\left(V,E\right)\in\mathcal{G}_{p,d}$
is generated and the Ising model is defined on it. Then, the spin
snapshots are obtained using Monte-Carlo sampling, yielding the dataset $\mathfrak{X}_{1}^{n}$.
The regularization parameter is set to be a constant factor of $\sqrt{\frac{\log p}{n}}$.
For any graph, we performed simulations using neighborhood-based Lasso (\ref{eq:lasso-estimator}) $\forall r \in V$ and then the associated signed neighborhood $\mathcal{\hat{N}}_{\pm}\left(r\right)$ is estimated as (\ref{eq:Neighbor-est}). Similar to \cite{ravikumar2010high},  the sample size $n$ scaling is set to be proportional to $d\log p$. For comparison, the results
of the $\ell_{1}$-LogR estimator \cite{ravikumar2010high} are also shown. The results are averaged over 200 trials in all cases. 

The results of RR graph and grid graph are shown in Figure \ref{fig:RR-result}.  In both cases, even for grid graph with many loops, using the Lasso estimator,
all curves for different 
model sizes $p$ line up with each other well, demonstrating that
for a graph with fixed degree $d$, the ratio $n/\log p$ controls
the success or failure of the Ising model selection. Importantly, the behavior of Lasso is about the same as  $\ell_1$-LogR.  
\begin{figure}[t!]
\centering{}\includegraphics[width=9cm]{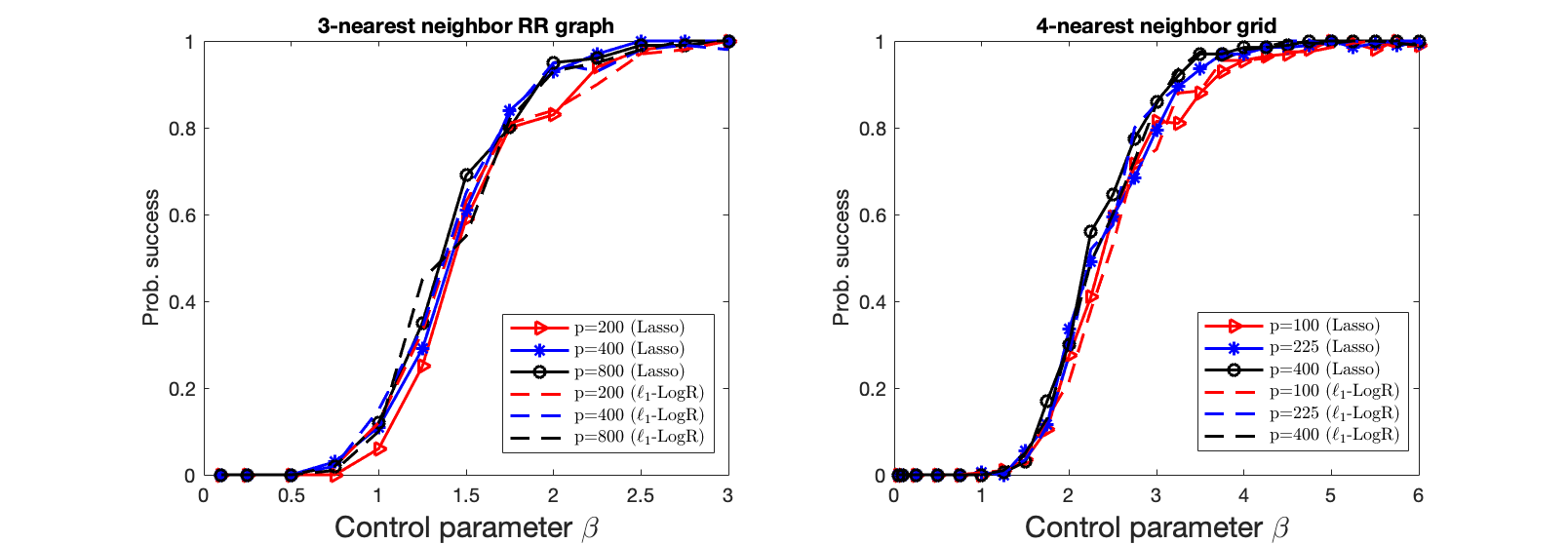}\caption{Success probability
versus the control parameter $\beta$ for Ising
models.  Left: RR graph with $d=3$ and mixed interactions
$\theta_{rt}^{*}=\pm0.4$ for all $\left(r,t\right)\in E$, $\beta=\frac{n}{10d\log p}$; Right: 4-nearest neighbor grid graph with $d=4$ and positive interactions
$\theta_{rt}^{*}=0.2$ for all $\left(r,t\right)\in E$, $\beta=\frac{n}{15d\log p}$. \label{fig:RR-result}}
\end{figure}

\begin{figure}[t!]
\centering{}\includegraphics[width=9cm]{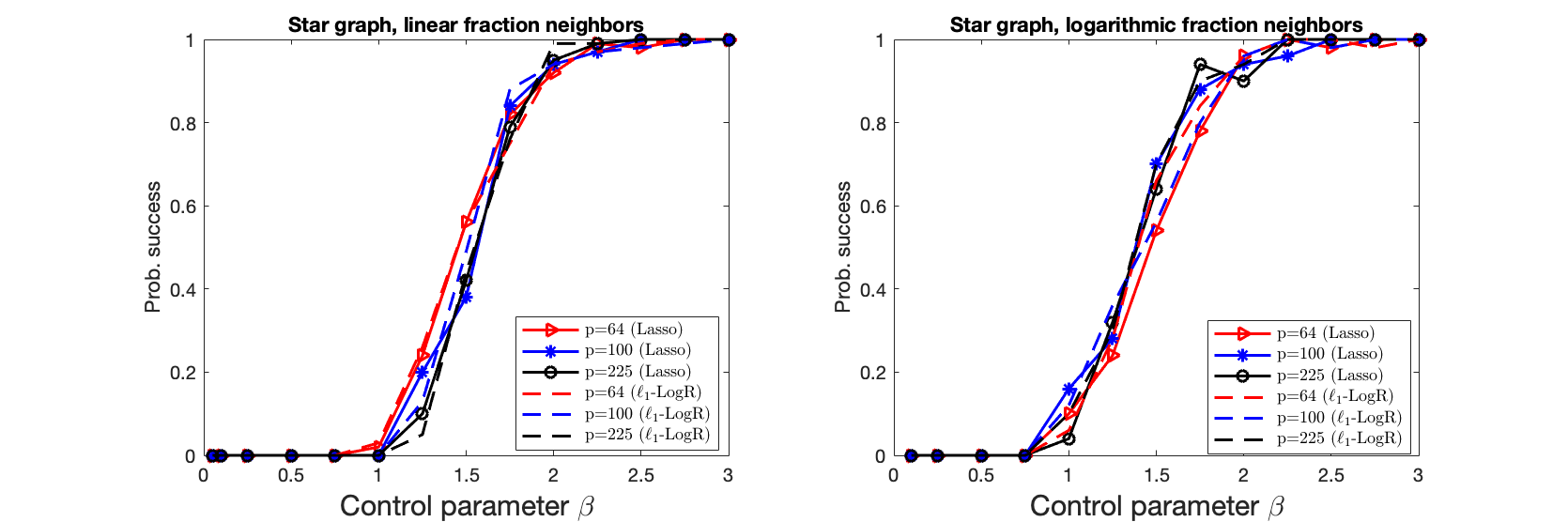}\caption{Success probability 
versus the control parameter $\beta=\frac{n}{10d\log p}$ for Ising
models on star-shaped graphs for attractive interactions $\theta_{rt}^{*}=\frac{ 1.2}{\sqrt{d}}$ for all $\left(r,t\right)\in E$.   Left: linear growth in degrees, i.e., $d = \left\lceil 0.1p \right\rceil $; Right: logarithmic growth in degrees, i.e., $d = \left\lceil \log{p} \right\rceil$. \label{fig:lin-star-result}}
\end{figure}

Figure \ref{fig:lin-star-result} shows results for star-shaped graph whose maximum degree $d$ is unbounded and grows as the dimension $p$ grows. Two kinds of star-shaped graphs are considered by designating one node as the hub and connecting it to $d<(p-1)$ of its neighbors. Specifically, for linear sparsity, it is assumed that $d = \left\lceil 0.1p \right\rceil$ while for logarithmic sparsity, we assume $d = \left\lceil \log{p} \right\rceil $. We use  positive interactions and set the active interactions to be $\theta_{rt}^{*}=\frac{1.2}{\sqrt{d}}$
for all $\left(r,t\right)\in E$ as \cite{ravikumar2010high}. As depicted in Figure   \ref{fig:lin-star-result},  Lasso behaves similarly as $\ell_1$-LogR in both cases, which is consistent with our theoretical analysis.

\section{\label{sec:conclusion}Conclusion}
We have theoretically analyzed the model selection consistency of Lasso, \textit{both with and without post-thresholding}, for the problem of  high-dimensional Ising model selection with a focus on the  paramagnetic phase. Specifically, in the case without post-thresholding, we prove that Lasso is model selection consistent with the same  order of sample complexity as that of $\ell_1$-LogR for RR graphs. For general tree-like graphs, similar result is obtained under mild assumptions of the \textit{dependency condition} and \textit{incoherence condition}. Moreover, in the case with post-thresholding, for general tree-like graphs, we not only obtain an upper  bound of the reconstructed square error of Lasso, but also prove the consistency of Lasso with post-thresholding with the same order of  sample complexity as that of $\ell_1$-LogR and RISE without any  assumptions on the \textit{dependency condition} and \textit{incoherence condition}. Experimental results are consistent with the theoretical analysis.

There are several interesting future directions for current study. First, since our focus in this paper is the paramagnetic phase, one important future work is to extend the current analysis to high-dimensional Ising models defined on general graphs beyond the paramagnetic phase, e.g., ferromagnetic phase, to see whether it still can, similarly as $\ell_1$-LogR and RISE, successfully recover the graph structure of Ising models with the same order of the number of samples.  Another future work is to investigate the performance of Lasso  for high-dimensional Ising model selection in the non-i.i.d. case \cite{dutt2021exponential}.  The study of other alternative simple and efficient methods for Ising model selection is also  an interesting topic for
future investigation.

\subsubsection*{Acknowledgements}
This work was supported by JSPS KAKENHI Nos. 17H00764, 18K11463, and 19H01812, 22H05117,
and JST CREST Grant Number JPMJCR1912, Japan.

\bibliography{mybib}

\begin{thebibliography}{}

\bibitem[Abbara et~al., 2020]{Abbara2019c}
Abbara, A., Kabashima, Y., Obuchi, T., and Xu, Y. (2020).
\newblock Learning performance in inverse {I}sing problems with sparse teacher
  couplings.
\newblock {\em Journal of Statistical Mechanics: Theory and Experiment},
  2020(7):073402.

\bibitem[Bachschmid-Romano and Opper, 2017]{bachschmid2017statistical}
Bachschmid-Romano, L. and Opper, M. (2017).
\newblock A statistical physics approach to learning curves for the inverse
  ising problem.
\newblock {\em Journal of Statistical Mechanics: Theory and Experiment},
  2017(6):063406.

\bibitem[Bento and Montanari, 2009]{bento2009graphical}
Bento, J. and Montanari, A. (2009).
\newblock Which graphical models are difficult to learn?
\newblock In {\em Proceedings of the 22nd International Conference on Neural
  Information Processing Systems}, pages 1303--1311.

\bibitem[Besag, 1975]{besag1975statistical}
Besag, J. (1975).
\newblock Statistical analysis of non-lattice data.
\newblock {\em Journal of the Royal Statistical Society: Series D (The
  Statistician)}, 24(3):179--195.

\bibitem[Bresler, 2015]{bresler2015efficiently}
Bresler, G. (2015).
\newblock Efficiently learning {I}sing models on arbitrary graphs.
\newblock In {\em Proceedings of the forty-seventh annual ACM symposium on
  Theory of computing}, pages 771--782.

\bibitem[Brillinger, 1982]{brillinger1982generalized}
Brillinger, D.~R. (1982).
\newblock A generalized linear model with {G}aussian regressor variables.
\newblock In {\em A Festschrift for Erich L. Lehmann}, page 97–114.

\bibitem[Decelle and Ricci-Tersenghi, 2014]{decelle2014pseudolikelihood}
Decelle, A. and Ricci-Tersenghi, F. (2014).
\newblock Pseudolikelihood decimation algorithm improving the inference of the
  interaction network in a general class of {I}sing models.
\newblock {\em Physical review letters}, 112(7):070603.

\bibitem[Dembo and Montanari, 2010]{dembo2010ising}
Dembo, A. and Montanari, A. (2010).
\newblock Ising models on locally tree-like graphs.
\newblock {\em The Annals of Applied Probability}, 20(2):565--592.

\bibitem[Dobriban and Wager, 2018]{dobriban2018high}
Dobriban, E. and Wager, S. (2018).
\newblock High-dimensional asymptotics of prediction: Ridge regression and
  classification.
\newblock {\em The Annals of Statistics}, 46(1):247--279.

\bibitem[Dutt et~al., 2021]{dutt2021exponential}
Dutt, A., Lokhov, A.~Y., Vuffray, M., and Misra, S. (2021).
\newblock Exponential reduction in sample complexity with learning of {I}sing
  model dynamics.
\newblock {\em arXiv preprint arXiv:2104.00995}.

\bibitem[Ekeberg et~al., 2013]{ekeberg2013improved}
Ekeberg, M., L{\"o}vkvist, C., Lan, Y., Weigt, M., and Aurell, E. (2013).
\newblock Improved contact prediction in proteins: using pseudolikelihoods to
  infer potts models.
\newblock {\em Physical Review E}, 87(1):012707.

\bibitem[Engel and Van~den Broeck, 2001]{engel2001statistical}
Engel, A. and Van~den Broeck, C. (2001).
\newblock {\em Statistical mechanics of learning}.
\newblock Cambridge University Press.

\bibitem[Erdogdu et~al., 2019]{erdogdu2019scalable}
Erdogdu, M.~A., Bayati, M., and Dicker, L.~H. (2019).
\newblock Scalable approximations for generalized linear problems.
\newblock {\em The Journal of Machine Learning Research}, 20(1):231--275.

\bibitem[Genzel, 2016]{genzel2016high}
Genzel, M. (2016).
\newblock High-dimensional estimation of structured signals from non-linear
  observations with general convex loss functions.
\newblock {\em IEEE Transactions on Information Theory}, 63(3):1601--1619.

\bibitem[Gomila, 2021]{gomila2021logistic}
Gomila, R. (2021).
\newblock Logistic or linear? {E}stimating causal effects of experimental
  treatments on binary outcomes using regression analysis.
\newblock {\em Journal of Experimental Psychology: General}, 150(4):700.

\bibitem[Hoeffding, 1994]{hoeffding1994probability}
Hoeffding, W. (1994).
\newblock Probability inequalities for sums of bounded random variables.
\newblock In {\em The Collected Works of Wassily Hoeffding}, pages 409--426.
  Springer.

\bibitem[H{\"o}fling and Tibshirani, 2009]{hofling2009estimation}
H{\"o}fling, H. and Tibshirani, R. (2009).
\newblock Estimation of sparse binary pairwise {M}arkov networks using
  pseudo-likelihoods.
\newblock {\em Journal of Machine Learning Research}, 10(4).

\bibitem[Ising, 1925]{ising1925beitrag}
Ising, E. (1925).
\newblock Beitrag zur theorie des ferromagnetismus.
\newblock {\em Zeitschrift f{\"u}r Physik}, 31(1):253--258.

\bibitem[Kappen and Rodr{\'\i}guez, 1998]{kappen1998efficient}
Kappen, H.~J. and Rodr{\'\i}guez, F. d.~B. (1998).
\newblock Efficient learning in {B}oltzmann machines using linear response
  theory.
\newblock {\em Neural Computation}, 10(5):1137--1156.

\bibitem[Koller and Friedman, 2009]{koller2009probabilistic}
Koller, D. and Friedman, N. (2009).
\newblock {\em Probabilistic graphical models: principles and techniques}.
\newblock MIT press.

\bibitem[Krishnan et~al., 2020]{krishnan2020modified}
Krishnan, J., Torabi, R., Schuppert, A., and Di~Napoli, E. (2020).
\newblock A modified {I}sing model of barab{\'a}si--albert network with
  gene-type spins.
\newblock {\em Journal of mathematical biology}, 81(3):769--798.

\bibitem[Liebl and Zacharias, 2021]{liebl2021accurate}
Liebl, K. and Zacharias, M. (2021).
\newblock Accurate modeling of dna conformational flexibility by a multivariate
  {I}sing model.
\newblock {\em Proceedings of the National Academy of Sciences}, 118(15).

\bibitem[Lokhov et~al., 2018]{lokhov2018optimal}
Lokhov, A.~Y., Vuffray, M., Misra, S., and Chertkov, M. (2018).
\newblock Optimal structure and parameter learning of {I}sing models.
\newblock {\em Science advances}, 4(3):e1700791.

\bibitem[Marbach et~al., 2012]{marbach2012wisdom}
Marbach, D., Costello, J.~C., K{\"u}ffner, R., Vega, N.~M., Prill, R.~J.,
  Camacho, D.~M., Allison, K.~R., Kellis, M., Collins, J.~J., and Stolovitzky,
  G. (2012).
\newblock Wisdom of crowds for robust gene network inference.
\newblock {\em Nature methods}, 9(8):796--804.

\bibitem[McAuley and Leskovec, 2012]{mcauley2012learning}
McAuley, J.~J. and Leskovec, J. (2012).
\newblock Learning to discover social circles in ego networks.
\newblock volume 2012, pages 548--56. Citeseer.

\bibitem[Meng et~al., 2020]{meng2020structure}
Meng, X., Obuchi, T., and Kabashima, Y. (2020).
\newblock Structure learning in inverse {I}sing problems using $\ell_2
  $-regularized linear estimator.
\newblock {\em arXiv preprint arXiv:2008.08342}.

\bibitem[Meng et~al., 2021]{meng2021ising}
Meng, X., Obuchi, T., and Kabashima, Y. (2021).
\newblock Ising model selection using $\ell_1$-regularized linear regression: A
  statistical mechanics analysis.
\newblock {\em Advances in Neural Information Processing Systems}, 34.

\bibitem[Mezard and Montanari, 2009]{mezard2009information}
Mezard, M. and Montanari, A. (2009).
\newblock {\em Information, physics, and computation}.
\newblock Oxford University Press.

\bibitem[Morcos et~al., 2011]{morcos2011direct}
Morcos, F., Pagnani, A., Lunt, B., Bertolino, A., Marks, D.~S., Sander, C.,
  Zecchina, R., Onuchic, J.~N., Hwa, T., and Weigt, M. (2011).
\newblock Direct-coupling analysis of residue coevolution captures native
  contacts across many protein families.
\newblock {\em Proceedings of the National Academy of Sciences},
  108(49):E1293--E1301.

\bibitem[Negahban et~al., 2012]{negahban2012unified}
Negahban, S.~N., Ravikumar, P., Wainwright, M.~J., Yu, B., et~al. (2012).
\newblock A unified framework for high-dimensional analysis of $ m $-estimators
  with decomposable regularizers.
\newblock {\em Statistical science}, 27(4):538--557.

\bibitem[Nguyen and Berg, 2012]{nguyen2012bethe}
Nguyen, H.~C. and Berg, J. (2012).
\newblock Bethe--{P}eierls approximation and the inverse {I}sing problem.
\newblock {\em Journal of Statistical Mechanics: Theory and Experiment},
  2012(03):P03004.

\bibitem[Nguyen et~al., 2017]{nguyen2017inverse}
Nguyen, H.~C., Zecchina, R., and Berg, J. (2017).
\newblock Inverse statistical problems: from the inverse {I}sing problem to
  data science.
\newblock {\em Advances in Physics}, 66(3):197--261.

\bibitem[Nishimori, 2001]{nishimori2001statistical}
Nishimori, H. (2001).
\newblock {\em Statistical physics of spin glasses and information processing:
  an introduction}.
\newblock Number 111. Clarendon Press.

\bibitem[Opper and Saad, 2001]{opper2001advanced}
Opper, M. and Saad, D. (2001).
\newblock {\em Advanced mean field methods: Theory and practice}.
\newblock MIT press.

\bibitem[Plan and Vershynin, 2016]{plan2016generalized}
Plan, Y. and Vershynin, R. (2016).
\newblock The generalized lasso with non-linear observations.
\newblock {\em IEEE Transactions on information theory}, 62(3):1528--1537.

\bibitem[Ravikumar et~al., 2010]{ravikumar2010high}
Ravikumar, P., Wainwright, M.~J., Lafferty, J.~D., et~al. (2010).
\newblock High-dimensional {I}sing model selection using $\ell_1$-regularized
  logistic regression.
\newblock {\em The Annals of Statistics}, 38(3):1287--1319.

\bibitem[Ricci-Tersenghi, 2012]{ricci2012bethe}
Ricci-Tersenghi, F. (2012).
\newblock The {B}ethe approximation for solving the inverse {I}sing problem: a
  comparison with other inference methods.
\newblock {\em Journal of Statistical Mechanics: Theory and Experiment},
  2012(08):P08015.

\bibitem[Rockafellar, 1970]{rockafellar1970convex}
Rockafellar, R.~T. (1970).
\newblock {\em Convex analysis}, volume~36.
\newblock Princeton university press.

\bibitem[Rothman et~al., 2008]{rothman2008sparse}
Rothman, A.~J., Bickel, P.~J., Levina, E., Zhu, J., et~al. (2008).
\newblock Sparse permutation invariant covariance estimation.
\newblock {\em Electronic Journal of Statistics}, 2:494--515.

\bibitem[Santhanam and Wainwright, 2012]{santhanam2012information}
Santhanam, N.~P. and Wainwright, M.~J. (2012).
\newblock Information-theoretic limits of selecting binary graphical models in
  high dimensions.
\newblock {\em IEEE Transactions on Information Theory}, 58(7):4117--4134.

\bibitem[Scikit-learn, ]{RidgeClassifier}
Scikit-learn.
\newblock Ridge classification.
\newblock
  \url{https://scikit-learn.org/stable/modules/linear_model.html#ridge-regression}.

\bibitem[Tanaka, 1998]{tanaka1998mean}
Tanaka, T. (1998).
\newblock Mean-field theory of {B}oltzmann machine learning.
\newblock {\em Physical Review E}, 58(2):2302.

\bibitem[Thrampoulidis et~al., 2015]{thrampoulidis2015lasso}
Thrampoulidis, C., Abbasi, E., and Hassibi, B. (2015).
\newblock Lasso with non-linear measurements is equivalent to one with linear
  measurements.
\newblock {\em Advances in Neural Information Processing Systems},
  28:3420--3428.

\bibitem[Tibshirani, 1996]{tibshirani1996regression}
Tibshirani, R. (1996).
\newblock Regression shrinkage and selection via the lasso.
\newblock {\em Journal of the Royal Statistical Society: Series B
  (Methodological)}, 58(1):267--288.

\bibitem[Vershynin, 2018]{vershynin2018high}
Vershynin, R. (2018).
\newblock {\em High-dimensional probability: An introduction with applications
  in data science}, volume~47.
\newblock Cambridge university press.

\bibitem[Vuffray et~al., 2016]{vuffray2016interaction}
Vuffray, M., Misra, S., Lokhov, A., and Chertkov, M. (2016).
\newblock Interaction screening: Efficient and sample-optimal learning of
  {I}sing models.
\newblock In {\em Advances in Neural Information Processing Systems}, pages
  2595--2603.

\bibitem[Wainwright, 2009]{Wainwright2009sharp}
Wainwright, M.~J. (2009).
\newblock Sharp thresholds for high-dimensional and noisy sparsity recovery
  using $\ell_1$-constrained quadratic programming (lasso).
\newblock {\em IEEE Transactions on Information Theory}, 55(5):2183--2202.

\bibitem[Wainwright and Jordan, 2008]{wainwright2008graphical}
Wainwright, M.~J. and Jordan, M.~I. (2008).
\newblock {\em Graphical models, exponential families, and variational
  inference}.
\newblock Now Publishers Inc.

\bibitem[Wainwright et~al., 2007]{wainwright2007high}
Wainwright, M.~J., Lafferty, J.~D., and Ravikumar, P.~K. (2007).
\newblock High-dimensional graphical model selection using $\ell_1$-regularized
  logistic regression.
\newblock In {\em Advances in neural information processing systems}, pages
  1465--1472.

\bibitem[Zhang et~al., 2016]{zhang2016consistency}
Zhang, Y., Guo, W., and Ray, S. (2016).
\newblock On the consistency of feature selection with lasso for non-linear
  targets.
\newblock In {\em International Conference on Machine Learning}, pages
  183--191. PMLR.

\bibitem[Zhao and Yu, 2006]{zhao2006model}
Zhao, P. and Yu, B. (2006).
\newblock On model selection consistency of lasso.
\newblock {\em The Journal of Machine Learning Research}, 7:2541--2563.

\end{thebibliography}

\appendix
\onecolumn

\section{\label{appendix:lemma-recaled-solution} Proof of Lemma \ref{lemma-recaled-solution}}
\begin{proof}
The gradient of the square loss $\ell\left(\theta_{\setminus r};\mathfrak{X}_{1}^{n}\right)$
in (\ref{eq:loss-def}) w.r.t. $\theta_{\setminus r}$ reads
\begin{equation}
\nabla\ell\left(\theta_{\setminus r};\mathfrak{X}_{1}^{n}\right)=\frac{1}{n}\sum_{i=1}^{n}x_{\setminus r}^{(i)}\left(x_{r}^{(i)}-\sum_{t\in V\setminus r}\theta_{rt}x_{t}^{(i)}\right).\label{eq:gradient-loss}
\end{equation}
After taking expectation of gradient $\nabla\ell\left(\theta_{\setminus r};\mathfrak{X}_{1}^{n}\right)$
over the distribution $\mathbb{P}_{\theta^{*}}\left(x\right)$ and
setting it to be zero, we obtain $\mathbb{E}_{\theta^{*}}\left(\nabla\ell\left(\theta_{\setminus r};\mathfrak{X}_{1}^{n}\right)\right)=0$
in matrix form:
\begin{align}
Q_{r}^{*}\theta_{\setminus r} & =b,\label{eq:Linear-equation}
\end{align}
where $Q_{r}^{*}=\mathbb{E}_{\theta^{*}}\left(X_{\setminus r}X_{r}^{T}\right)$
is the covariance matrix of $X_{\setminus r}$ and $b=\mathbb{E}_{\theta^{*}}\left(X_{\setminus r}X_{r}\right)$.
The solution to (\ref{eq:Linear-equation}), denoted as $\tilde{\theta}_{\setminus r}^{*}$, can be analytically obtained as $\tilde{\theta}_{\setminus r}^{*}=\left(Q_{r}^{*}\right)^{-1}b$.
Next, we construct the full covariance matrix $C=\mathbb{E}_{\theta^{*}}\left(XX^{T}\right)$
of all spins $X$ as follows
\begin{equation}
C=\left[\begin{array}{cc}
1 & b^{T}\\
b & Q_{r}^{*}
\end{array}\right],\label{eq:Covariance-matrix}
\end{equation}
where $X_{r}$ is indexed as the first variable in $C$ without loss
of generality. From the block matrix inversion
lemma, the inverse covariance matrix can be computed as
\begin{align}
C^{-1} =\left[\begin{array}{cc}
F_{11}^{-1} & -F_{11}^{-1}\left(\tilde{\theta}_{\setminus r}^{*}\right)^{T}\\
-\tilde{\theta}_{\setminus r}^{*}F_{11}^{-1} & F_{22}^{-1}
\end{array}\right],\label{eq:Covariance-matrix-1}
\end{align}
where
\begin{align}
F_{11} & =1-b^{T}\left(Q_{r}^{*}\right)^{-1}b,\\
F_{22} & =Q_{r}^{*}-bb^{T}.
\end{align}
On the other hand, for general tree-like graphs in the paramagnetic
phase, the inverse covariance matrix $C^{-1}$  can be computed from the Hessian
of the Gibbs free energy \cite{ricci2012bethe, nguyen2012bethe,Abbara2019c}. Specifically, each element of the covariance matrix  $C=\{C_{rt}\}_{r,t\in V}$ can be expressed as 
\begin{equation}
C_{rt} = \mathbb{E}_{\theta^*}(x_r x_t) -\mathbb{E}_{\theta^*}(x_r) \mathbb{E}_{\theta^*}(x_t)
= \frac{\partial^2 \log Z({\sigma}) }
{\partial \sigma_r \partial \sigma_t}, 
\end{equation}
where 
$Z({\sigma}) = \sum_{x} \mathbb{P}_{\theta^{*}}\left(x\right)
e^{\sum_{s\in V} \sigma_s x_s}$ with $\sigma=\{\sigma_s\}_{s\in V}$ and the assessment is carried out at ${\sigma} = {0}$. In addition, for technical convenience we introduce the Gibbs free energy  as
\begin{equation}
A\left({m}\right)=\underset{{\sigma}}{\max}\left\{ {\sigma}^{T}{m}-\log Z\left({\sigma}\right)\right\}.\label{eq:Gibbs-Gm-1}
\end{equation}
The definition of (\ref{eq:Gibbs-Gm-1}) indicates that following two  relations hold: 
\begin{align}
\frac{\partial m_r}{\partial \sigma_t} 
=\frac{\partial^2 \log Z({\sigma})}{\partial \sigma_r \partial \sigma_t} = C_{rt}, \\
\frac{\partial \sigma_r}{\partial m_t} =[C^{-1}]_{rt} 
= \frac{\partial^2 A({m})}{\partial m_r
\partial m_t},
\end{align}
where the evaluations are performed at ${\sigma} = {0}$ and ${m} = \arg\min _{{m}} A({m})$ ($={0}$ under the paramagnetic assumption). 
Consequently, the inverse covariance matrix of a tree-like graph $G\in \mathcal{G}_{p,d}$ can be computed
as \cite{ricci2012bethe, nguyen2012bethe,Abbara2019c}
\begin{align}
\left[C^{-1}\right]_{rt} & =\left(\sum_{u\in\mathcal{N}\left(r\right)}\frac{1}{1-\tanh^{2}\left(\theta_{ru}^{*}\right)}-d_{r}+1\right)\delta_{rt} \nonumber\\
&-\frac{\tanh\left(\theta_{rt}^{*}\right)}{1-\tanh^{2}\left(\theta_{rt}^{*}\right)}\left(1-\delta_{rt}\right).\label{eq:C0-inverse-matrix-1-v2}
\end{align}

The two representations of $C^{-1}$ in (\ref{eq:Covariance-matrix-1})
and (\ref{eq:C0-inverse-matrix-1-v2}) are equivalent so that the
corresponding elements should equal to each other. Thus, the
following identities hold
\begin{align}
\begin{cases}
F_{11}^{-1}=\sum_{u\in\mathcal{N}\left(r\right)}\frac{1}{1-\tanh^{2}\left(\theta_{ru}^{*}\right)}-d_{r}+1,\\
\tilde{\theta}_{\setminus r}^{*}F_{11}^{-1}=\frac{\tanh\left(\theta_{\setminus r}^{*}\right)}{1-\tanh^{2}\left(\theta_{\setminus r}^{*}\right)},
\end{cases}\label{eq:Jbar-condition}
\end{align}
where $\tanh\left(\cdot\right)$ is applied element-wise. From (\ref{eq:Jbar-condition}),
we obtain (\ref{eq:rescaled-result}), which is a rescaled version of the true interactions.
In particular, for RR graphs with constant coupling $\theta_{rt}^{*}=\theta_{0},\forall\left(r,t\right)\in E$
and $d_{r}=d$, substituting the results one can obtain
\begin{equation}
\tilde{\theta}_{rt}^{*}=\begin{cases}
\frac{\tanh\left(\theta_{0}\right)}{1+\left(d-1\right)\tanh^{2}\left(\theta_{0}\right)} & \textrm{if}\left(r,t\right)\in E;\\
0 & \textrm{otherwise.}
\end{cases}\label{eq:rescaled-result-RR}
\end{equation}
which completes the proof. 
\end{proof}

\section{{\label{appendix:RR-C1C2-proof} Proof of Lemma \ref{lem:RR-graph-(C1)(C2)}}}
The corresponding belief propagation
(BP) equation on a RR graph  can be written as follows \cite{mezard2009information}
\begin{equation}
m_{r\to t}=\tanh\left(\sum_{k\in\text{\ensuremath{\mathcal{N}\left(r\right)}}\backslash t}\tanh^{-1}\left(\tanh\left(\theta_{0}\right)m_{k\to r}\right)\right).
\end{equation}
where $m_{r\to t}$ is the message from node $r$ to node $t$. The
spontaneous magnetization for the node $r\in V$ is assessed as
\begin{equation}
m_{r}=\tanh\left(\sum_{t\in\text{\ensuremath{\mathcal{N}\left(r\right)}}}\tanh^{-1}\left(\tanh\left(\theta_{0}\right)m_{t\to r}\right)\right).
\end{equation}

Due to the uniformity of RR graphs, these equations are reduced to
\begin{align}
m_{c} & =\tanh\left(\left(d-1\right)\tanh^{-1}\left(\tanh\left(\theta_{0}\right)m_{c}\right)\right),\\
m & =\tanh\left(d\tanh^{-1}\left(\tanh\left(\theta_{0}\right)m_{c}\right)\right),
\end{align}
where we set $m_{r\to t}\coloneqq m_{c}$ and $m_{r}\coloneqq m$ for all
directed edges $r\to t$ and all nodes $r\in V$.

Suppose that $x=(x_{r})_{r=1}^{p}$ is subject to a Hamiltonian $H\left(x\right)=-\sum_{s\neq t}\theta_{rt}^{*}x_{r}x_{t}$.
For this, we define the Helmholtz free energy as
\begin{equation}
F\left(\xi\right)=-\ln\left(\sum_{x}\exp\left(-H\left(x\right)+\sum_{r=1}^{p}\xi_{r}x_{r}\right)\right).
\end{equation}

Using $F\left(\xi\right)$, one can evaluate the expectation
as
\begin{equation}
m_r \coloneqq \ensuremath{\mathbb{E}_{\theta^{*}}\left\{ x_{r}\right\} }=-\left.\frac{\partial F\left(\xi\right)}{\partial\xi_{r}}\right|_{\xi=0}=\frac{\sum_{x}x_{r}\exp\left(-H\left(x\right)\right)}{\sum_{x}\exp\left(-H\left(x\right)\right)}.
\end{equation}

In addition, the covariance of $x_{r}$ and $x_{t}$ can be computed
as
\begin{align}
&\mathbb{E}_{\theta^{*}}\left\{ x_{r}x_{t}\right\} -\mathbb{E}_{\theta^{*}}\left\{ x_{r}\right\} \mathbb{E}_{\theta^{*}}\left\{ x_{t}\right\} = \left.\frac{\partial\ensuremath{\mathbb{E}_{\theta^{*}}\left\{ x_{r}\right\} }}{\partial\xi_{t}}\right|_{\xi=0} \nonumber \\
=&\frac{\sum_{x}x_{r}x_{t}\exp\left(-H\left(x\right)\right)}{\sum_{x}\exp\left(-H\left(x\right)\right)} -\frac{\sum_{x}x_{r}\exp\left(-H\left(x\right)\right)}{\sum_{x}\exp\left(-H\left(x\right)\right)}\cdot\frac{\sum_{x}x_{t}\exp\left(-H\left(x\right)\right)}{\sum_{x}\exp\left(-H\left(x\right)\right)},\label{eq:linear-response}
\end{align}
where the last equation is termed the \textit{linear response relation \cite{nishimori2001statistical}.}

Suppose that node $r$ is placed at the distance of $l$ from node
$t$. A remarkable property of tree-like graphs, including typical RR graphs, is that a unique path is
defined between two arbitrary nodes. This indicates that the linear
response relation (\ref{eq:linear-response}) can be evaluated by
the chain rule of partial derivative using messages of belief propagation
as 
\begin{align}
&\mathbb{E}_{\theta^{*}}\left\{ x_{r}x_{t}\right\} -\mathbb{E}_{\theta^{*}}\left\{ x_{r}\right\} \mathbb{E}_{\theta^{*}}\left\{ x_{t}\right\} =\left.\frac{\partial m_{r}}{\partial\xi_{t}}\right|_{\xi=0} \nonumber \\
&=(1-m^{2})\left(\frac{\tanh\left(\theta_{0}\right)(1-m_{c}^{2})}{1-\tanh^{2}\left(\theta_{0}\right)m_{c}^{2}}\right)^{l}.\label{eq:cov-result-general}
\end{align}

In the the paramagnetic phase where  $m=0$ and $m_{c}=0$, we have  
\begin{equation}
\mathbb{E}_{\theta^{*}}\left\{ x_{r}x_{t}\right\} -\mathbb{E}_{\theta^{*}}\left\{ x_{r}\right\} \mathbb{E}_{\theta^{*}}\left\{ x_{t}\right\} =\tanh^{l}\left(\theta_{0}\right).\label{eq:cov-result}
\end{equation}

Let us examine the \textit{dependency condition} (C1).  Since the distances between any two different nodes in  $S\coloneqq\left\{ \left(r,t\right)\mid t\in\mathcal{N}\left(r\right)\right\}$
are 2, all the off-diagonal elements in sub-matrix $Q_{SS}^{*}$ equal to $\tanh^2{\theta_{0}}$ and all the diagonal elements  equal to 1, i.e., 
\begin{align}
&Q_{SS}^{*} = \nonumber \\
&\left[\begin{array}{ccccc}
1 & \tanh^{2}\theta_{0} & \tanh^{2}\theta_{0} & \cdots & \tanh^{2}\theta_{0}\\
\tanh^{2}\theta_{0} & 1 & \tanh^{2}\theta_{0} & \vdots & \tanh^{2}\theta_{0}\\
\tanh^{2}\theta_{0} & \tanh^{2}\theta_{0} & \ddots & \tanh^{2}\theta_{0} & \vdots\\
\vdots & \cdots & \tanh^{2}\theta_{0} & 1 & \tanh^{2}\theta_{0}\\
\tanh^{2}\theta_{0} & \tanh^{2}\theta_{0} & \cdots & \tanh^{2}\theta_{0} & 1
\end{array}\right]_{d\times d}.  \label{eq:Qss_structure}
\end{align}
It can be analytically computed that $Q_{SS}^{*}$ has two different eigenvalues: one is $1+(d-1)\tanh^{2}\theta_{0}$ and the other is $1 - \tanh^{2}\theta_{0}$ with multiplicity $(d-1)$. Consequently, $Q_{SS}^{*}$ has bounded eigenvalue and we explicitly obtain  the result of $C_{\min}$ as 
\begin{align}
\Lambda_{\min}\left(Q_{SS}^{*}\right) = 1 - \tanh^{2}\theta_{0} \coloneqq C_{\min}.  \label{eq:eigen-upper-lower-1}
\end{align}
Then, we prove that the \textit{incoherence condition} (C2) also satisfies. From (\ref{eq:Qss_structure}), the inverse matrix $\left(Q_{SS}^{*}\right)^{-1}$ can be analytically computed as 
\begin{align}
    \left(Q_{SS}^{*}\right)^{-1} = \left[\begin{array}{ccccc}
a & b & b & \cdots & b\\
b & a & b & \vdots & b\\
b & b & \ddots & b & \vdots\\
\vdots & \cdots & b & a & b\\
b & b & \cdots & b & a
\end{array}\right]_{d\times d},\label{eq: Q_inv}
\end{align}
where 
\begin{align}
a	&=\frac{1+\left(d-2\right)\tanh^{2}\theta_{0}}{\left(1-\tanh^{2}\theta_{0}\right)\left(1+\left(d-1\right)\tanh^{2}\theta_{0}\right)}, \\
b	&=-\frac{\tanh^{2}\theta_{0}}{\left(1-\tanh^{2}\theta_{0}\right)\left(1+\left(d-1\right)\tanh^{2}\theta_{0}\right)}.
\end{align}
Then, by definition of $\interleave Q_{S^{c}S}^{*}\left(Q_{SS}^{*}\right)^{-1}\interleave_{\infty}$, it is achieved for $r\in S^{c}$ where $r$ belongs to the  nearest neighbors of the nodes in $S$. Specifically, in that case, the elements in the row in  $Q_{S^{c}S}^{*}$  associated with node $r\in S^{c}$ can only take two different values: one element is $\tanh{\theta_{0}}$ and the other ($d-1$) elements are $\tanh^{3}\theta_{0}$. Then, from (\ref{eq: Q_inv}), after some algebra, it can be calculated that 
\begin{align}
\interleave Q_{S^{c}S}^{*}\left(Q_{SS}^{*}\right)^{-1}\interleave_{\infty} = \tanh\theta_{0} \coloneqq 1 - \alpha,
\end{align}
where we obtain an analytical result $\alpha \coloneqq 1 - \tanh\theta_{0} \in(0,1]$, which completes the proof.

\section{Proofs of the key results}
\subsection{\label{appendix:lemma-Z_s_bounded_zeromean_var} Proof of Lemma \ref{lem: Ws-Zs-results}}
\begin{proof}
The result that $\mathbb{E}_{\theta^{*}}\left(Z_{s}^{\left(i\right)}\right)=0$ can be readily obtained by the definition of  $\tilde{\theta}_{\setminus r}^{*}$ in Lemma \ref{lemma-recaled-solution}.
Thus, to prove $\mathtt{Var}\left(Z_{s}^{\left(i\right)}\right)\leq1$, it suffices to prove $\mathbb{E}_{\theta^{*}}\left(\left(Z_{s}^{\left(i\right)}\right)^{2}\right)\leq1$ in the paramagnetic phase.

We introduce an auxiliary function
\begin{equation}
f_{1}\left(\theta_{\setminus r}\right)=\mathbb{E}_{\theta^{*}}\left(x_{r}^{(i)}-\sum_{t\in V\setminus r}\theta_{t}x_{t}^{(i)}\right)^{2}. \label{eq:fr_func}
\end{equation}
Thus we have $\mathbb{E}_{\theta^{*}}\left(\left(Z_{s}^{\left(i\right)}\right)^{2}\right)=f_{1}\left(\tilde{\theta}_{\setminus r}^{*}\right)$.
The gradient vector can be computed as $\nabla f_{1}\left(\theta_{\setminus r}\right)=2\mathbb{E}_{\theta^{*}}\left(\nabla\ell\left(\theta_{\setminus r};\mathfrak{X}_n\right)\right)$.
Since $\mathbb{E}_{\theta^{*}}\left(\nabla\ell\left(\tilde{\theta}_{\setminus r}^{*};\mathfrak{X}_n\right)\right)=0$
as shown in Lemma \ref{lemma-recaled-solution}, we have $\nabla f_{1}\left(\tilde{\theta}_{\setminus r}^{*}\right)=0$.
Moreover, since $\nabla^{2}f_{1}\left(\theta_{\setminus r}\right)=2\mathbb{E}_{\theta^{*}}\left(X_{\setminus r}X^T_{\setminus r}\right)\succ0$, we can conclude that $f_{1}\left(\theta_{\setminus r}\right)$ reaches
its minimum at $\theta_{\setminus r}=\tilde{\theta}_{\setminus r}^{*}$. 
As a result, we have 
\begin{align}
\mathbb{E}_{\theta^{*}}\left(\left(Z_{s}^{\left(i\right)}\right)^{2}\right)= & f_{1}\left(\theta_{\setminus r}=\tilde{\theta}_{\setminus r}^{*}\right)\nonumber \\
\leq & f_{1}\left(\theta_{\setminus r}=0\right) \nonumber  \\
= & \mathbb{E}_{\theta^{*}}\left(x_{r}^{(i)}\right)^{2} \nonumber \\
= & 1,
\end{align}
where in the last line the fact that $x_{r}^{(i)}\in\left\{ -1,+1\right\} ,\forall r\in V$ is used. Therefore,
we obtain $\mathtt{Var}\left(Z_{s}^{\left(i\right)}\right)\leq1$. 

Moreover, the  absolute value $\left|Z_{s}^{\left(i\right)}\right|$ is bounded. Specifically, 
(a) for RR graphs, in the paramagnetic phase, we have 
\begin{align}
    \left|Z_{s}^{\left(i\right)}\right| &=  \left|x_{s}^{(i)}(x_{r}^{(i)}-\sum_{t\in V\setminus r}\tilde{\theta}_{rt}^{*}x_{t}^{(i)})\right| \nonumber\\
    & \leq   1 + \sum_{t\in V\setminus r}|\tilde{\theta}_{rt}^{*}| \nonumber \\
  & =  1 + \frac{d\tanh\left(\theta_{0}\right)}{1+\left(d-1\right)\tanh^{2}\left(\theta_{0}\right)} \nonumber \\
   & \leq 2. \label{Z_s-bound-para}
\end{align}

(b) for general tee-like graphs, recalling the result (\ref{eq:rescaled-result}), we have
\begin{align}
&\left(\sum_{u\in\mathcal{N}\left(r\right)}\frac{1}{1-\tanh^{2}\left(\theta_{ru}^{*}\right)}-d_{r}+1\right)\sum_{t\in V\setminus r}\left|\tilde{\theta}_{rt}^{*}\right| \nonumber  \\ =&\sum_{t\in\mathcal{N}\left(r\right)}\frac{\left|\tanh\left(\theta_{rt}^{*}\right)\right|}{1-\tanh^{2}\left(\theta_{rt}^{*}\right)} \nonumber \\
  =&\sum_{t\in\mathcal{N}\left(r\right)}\frac{\left|\tanh\left(\theta_{rt}^{*}\right)\right|+1-\tanh^{2}\left(\theta_{rt}^{*}\right)+\tanh^{2}\left(\theta_{rt}^{*}\right)-1}{1-\tanh^{2}\left(\theta_{rt}^{*}\right)}\nonumber\\
 =&-d_{r}+\sum_{t\in\mathcal{N}\left(r\right)}\frac{\left|\tanh\left(\theta_{rt}^{*}\right)\right|+1-\tanh^{2}\left(\theta_{rt}^{*}\right)}{1-\tanh^{2}\left(\theta_{rt}^{*}\right)}, \label{eq:Bound-temp-result}
\end{align}
To proceed, consider an auxiliary function $f_2\left(x\right)=x+1-x^{2},0\leq x\leq1$.
Then it can be proved that $1\leq f_2\left(x\right)\leq\frac{5}{4}$,
so that from (\ref{eq:Bound-temp-result}), we have
\begin{equation}
\sum_{t\in V\setminus r}\left|\tilde{\theta}_{rt}^{*}\right|\leq\frac{-d_{r}+\frac{5}{4}\sum_{u\in\mathcal{N}\left(r\right)}\frac{1}{1-\tanh^{2}\left(\theta_{ru}^{*}\right)}}{\sum_{u\in\mathcal{N}\left(r\right)}\frac{1}{1-\tanh^{2}\left(\theta_{ru}^{*}\right)}-d_{r}+1}.\label{eq:Bound-temp-result1}
\end{equation}
It can be easily checked that $\sum_{u\in\mathcal{N}\left(r\right)}\frac{1}{1-\tanh^{2}\left(\theta_{ru}^{*}\right)}\in[d_{r},\infty)$.
We introduce another auxiliary function
\begin{equation}
f_3\left(x\right)=\frac{-d_{r}+\frac{5}{4}x}{x-d_{r}+1},x\in[d_{r},\infty).
\end{equation}
The first-order derivative of $f_3\left(x\right)$ can be easily computed
as
\begin{equation}
f_3^{'}\left(x\right)=\frac{5-d_{r}}{4\left(x-d_{r}+1\right)^{2}}.\label{eq:f-derivative-1}
\end{equation}
As a result, $f_3^{'}\left(x\right)>0$ when $d_{r}<5$ and $f_3^{'}\left(x\right)<0$
when $d_{r}>5$. Consequently,
\begin{equation}
\underset{x\in[d_r,\infty)}{\max}f_3\left(x\right)=\begin{cases}
\frac{5}{4} & d_{r}\leq5\\
\frac{d_{r}}{4} & d_{r}>5
\end{cases}
\end{equation}
Finally, combining the above results together yields 
\begin{equation}
\left|Z_{s}^{\left(i\right)}\right|  \leq\max\left\{ \frac{9}{4},\frac{4+d_{r}}{4}\right\}  <d_{r},\forall d_{r}\geq3.
\end{equation}
By definition, there is $d_{r}\leq d$ so that $\left|Z_{s}^{\left(i\right)}\right|\leq d$,
which completes the proof. 
\end{proof}

\subsection{\label{appendix:lemma-W_inf_norm} Proof of Lemma \ref{lem:W_inf_norm}}
\begin{proof}
Frist, we prove the case (a). In this case, According to Lemma \ref{lem: Ws-Zs-results}, $\mathbb{E}_{\theta^{*}}\left(Z_{s}^{\left(i\right)}\right)=0$ and $\left|Z_{s}^{\left(i\right)}\right|\leq 2$, so that by the Azuma Hoeffding inequality \cite{vershynin2018high}, for $\forall\eta>0$, we have
\begin{equation}
\mathbb{P}\left(\left|W_{s}^{n}\right|>\eta\right)\leq2\exp\left(-\frac{\eta^{2}n}{8}\right).\label{eq:Wns_inequality-Hoef}
\end{equation}
Setting $\eta = \frac{\alpha\lambda_n}{2(2-\alpha)}$, we 
obtain 
\begin{equation}
\mathbb{P}\left(\frac{2-\alpha}{\lambda_{n}}\left|W_{s}^{n}\right|>\frac{\alpha}{2}\right) \leq 2\exp\left(-\frac{\alpha^{2}\lambda_{n}^2 n}{32(2-\alpha)^2}\right). 
\end{equation}
Then, by using a union bound we have
\begin{equation}
\mathbb{P}\left(\frac{2-\alpha}{\lambda_{n}}\left\Vert W^{n}\right\Vert _{\infty}\geq\frac{\alpha}{2}\right)  \leq 2\exp\left(-\frac{\alpha^{2}\lambda_{n}^2 n}{32(2-\alpha)^2} + \log p\right),
\end{equation}
which completes the proof of (a). 

In the case (b) for general graphs, the proof is slightly complicated. According to Lemma \ref{lem: Ws-Zs-results}, applying the
Bernstein's inequality \cite{vershynin2018high}, $\forall\eta>0$ we have 
\begin{equation}
\mathbb{P}\left(\left|W_{s}^{n}\right|>\eta\right)\leq2\exp\left(-\frac{\frac{1}{2}\eta^{2}n}{1+\frac{1}{3}d\eta}\right).\label{eq:Wns_inequality}
\end{equation}
Similar to \cite{vuffray2016interaction}, inverting the following
relation
\begin{equation}
\xi=\text{\ensuremath{\frac{\frac{1}{2}\eta^{2}n}{1+\frac{1}{3}d\eta}}}\label{eq:s-invert},
\end{equation}
and substituting the result in (\ref{eq:Wns_inequality}) yields
\begin{equation}
\mathbb{P}\left(\left|W_{s}^{n}\right|>\frac{1}{3}\left(u+\sqrt{u^{2}+18\frac{u}{d}}\right)\right)\leq2\exp\left(-\xi\right),\label{eq:Wns_inequality-1}
\end{equation}
where $u=\frac{\xi}{n}d$. Suppose that $n\geq\xi d^{2}$, then $u^{2}=\frac{\xi^{2}}{n^{2}}d^{2}\leq\frac{\xi}{n}$
while $\frac{u}{d}=\frac{\xi}{n}$. Consequently, we have 
\begin{align}
\frac{1}{3}\left(u+\sqrt{u^{2}+18\frac{u}{d}}\right) & \leq\frac{1}{3}\left(\sqrt{\frac{\xi}{n}}+\sqrt{\frac{\xi}{n}+18\frac{\xi}{n}}\right)\\
 & \leq\frac{1}{3}\left(\sqrt{\frac{\xi}{n}}+\sqrt{\frac{\xi}{n}}\sqrt{25}\right)\\
 & =2\sqrt{\frac{\xi}{n}},
\end{align}
where a relaxed result is obtained. Subsequently, we obtain an expression
which is independent of $d$: 
\begin{equation}
\mathbb{P}\left(\left|W_{s}^{n}\right|>2\sqrt{\frac{\xi}{n}}\right)\leq2\exp\left(-\xi\right).\label{eq:Wns_inequality-1-1}
\end{equation}
Setting $\xi=\left(c+1\right)\log p$, then if $\lambda_{n}\geq\frac{4\left(2-\alpha\right)\sqrt{c+1}}{\alpha}\sqrt{\frac{\log p}{n}}$,
we have $\frac{\alpha\lambda_{n}}{2\left(2-\alpha\right)}\geq2\sqrt{\frac{\xi}{n}}$
so that 
\begin{align}
&\mathbb{P}\left(\frac{2-\alpha}{\lambda_{n}}\left|W_{s}^{n}\right|>\frac{\alpha}{2}\right) \leq\mathbb{P}\left(\left|W_{s}^{n}\right|>2\sqrt{\frac{\xi}{n}}\right)  \nonumber \\
 &\leq2\exp\left(-\left(c+1\right)\log p\right). 
\end{align}
Then, by using a union bound we have
\begin{align}
\mathbb{P}\left(\frac{2-\alpha}{\lambda_{n}}\left\Vert W^{n}\right\Vert _{\infty}\geq\frac{\alpha}{2}\right) \leq2\exp\left(-c\log p\right).
\end{align}
As a result, when $n\geq\left(c+1\right)d^{2}\log p$, as long as
${\lambda_{n}\geq\frac{4\sqrt{c+1}\left(2-\alpha\right)}{\alpha}\sqrt{\frac{\log p}{n}}}$,
it is guaranteed that $\mathbb{P}\left(\frac{2-\alpha}{\lambda_{n}}\left\Vert W^{n}\right\Vert _{\infty}\geq\frac{\alpha}{2}\right)\rightarrow0$
at rate $\exp\left(-c\log p\right)$ for some constant $c>0$, which
completes the proof. 
\end{proof}

\subsection{\label{appendix:lem:L2-consistency} Proof of Lemma \ref{lem:L2-consistency}}
\begin{proof}
Using the method in \cite{rothman2008sparse}, here the proof follows
\cite{ravikumar2010high} but with essential modifications.
First, define a function $\mathbb{R}^{d}\rightarrow\mathbb{R}$
as follows \cite{rothman2008sparse}
\begin{align}
G\left(u_{S}\right)\coloneqq & \ell\left(\tilde{\theta}_{S}^{*}+u_{S};\mathfrak{X}_n\right)-\ell\left(\tilde{\theta}_{S}^{*};\mathfrak{X}_n\right) \nonumber \\ &+\lambda_{n}\left(\left\Vert \tilde{\theta}_{S}^{*}+u_{S}\right\Vert _{1}-\left\Vert \tilde{\theta}_{S}^{*}\right\Vert _{1}\right).\label{eq:G-fun-def}
\end{align}
Note that $G$ is a convex function w.r.t. $u_S$. Then $\hat{u}_{S}=\hat{\theta}_{S}-\tilde{\theta}_{S}^{*}$ minimizes
$G$ according to the definition in (\ref{eq:lasso-estimator}). Moreover,
it is easily seen that $G\left(0\right)=0$ so that $G\left(\hat{u}_{S}\right)\leq0$.
As described in \cite{ravikumar2010high}, if we can show that
there exists some radius $B>0$ and any  $u_{S}\in\mathbb{R}^{d}$ with $\left\Vert u_{S}\right\Vert _{2}=B$ satisfies $G(u_S)>0$,  
then we can claim that $\left\Vert \hat{u}_{S}\right\Vert _{2}\leq B$
since otherwise one can always, by appropriately choosing $t\in(0,1]$, find a convex combination $t\hat{u}_{S}+\left(1-t\right)0$ which
lies on the boundary of the ball with radius $B$ and thus $G\left(t\hat{u}_{S}+\left(1-t\right)0\right)\leq0$,
leading to contradiction. Consequently, it suffices to establish the strict
positivity of $G$ on the boundary of a ball with radius $B=M\lambda_{n}\sqrt{d}$, 
where $M>0$ is one parameter to choose later. 

Specifically, let $u_{S}\in\mathbb{R}^{d}$ be an arbitrary vector
with $\left\Vert u_{S}\right\Vert _{2}=B$. Expanding the quadratic form
$\ell\left(\tilde{\theta}_{S}^{*}+u_{S};\mathfrak{X}_n\right)$,
we have
\begin{align}
G\left(u_{S}\right)=& -\left(W_{S}^{n}\right)^{T}u_{S}+u_{S}^{T}Q_{SS}^{n}u_{S} \nonumber \\
&+\lambda_{n}\left(\left\Vert \tilde{\theta}_{S}^{*}+u_{S}\right\Vert _{1}-\left\Vert \tilde{\theta}_{S}^{*}\right\Vert _{1}\right),\label{eq:G-Taylor}
\end{align}
where $W_{S}^{n}$ is the sub-vector of $W^{n}=-\nabla\ell\left(\tilde{\theta}^{*};\mathfrak{X}_n\right)$, and 
$Q_{SS}^{n}$ is the sub-matrix of the sample matrix $Q^{n}$. The expression (\ref{eq:G-Taylor}) is simpler than the counterpart in \cite{ravikumar2010high}  which is obtained from the Taylor series expansion of the non-quadratic loss function and thus its quadratic term is dependent on $\theta$. 
To proceed, we investigate the bounds
of the three terms in the right hand side (RHS) of (\ref{eq:G-Taylor}), respectively. 

Since $\left\Vert u_{S}\right\Vert _{1}\leq\sqrt{d}\left\Vert u_{S}\right\Vert _{2}$
and $\left\Vert W_{S}^{n}\right\Vert _{\infty}\leq\frac{\lambda_{n}}{2}$,
the first term is bounded as
\begin{align}
\left|-\left(W_{S}^{n}\right)^{T}u_{S}\right|&\leq\left\Vert W_{S}^{n}\right\Vert _{\infty}\left\Vert u_{S}\right\Vert _{1}\leq\left\Vert W_{S}^{n}\right\Vert _{\infty}\sqrt{d}\left\Vert u_{S}\right\Vert _{2}\nonumber \\
&\leq\left(\lambda_{n}\sqrt{d}\right)^{2}\frac{M}{2}.\label{eq:first-term-bound}
\end{align}
The third term is bounded as
\begin{align}
& \lambda_{n}\left(\left\Vert  \tilde{\theta}_{S}^{*}+u_{S}\right\Vert _{1}-\left\Vert \tilde{\theta}_{S}^{*}\right\Vert _{1}\right)\nonumber \\
&\geq-\lambda_{n}\left\Vert u_{S}\right\Vert _{1}\geq-\lambda_{n}\sqrt{d}\left\Vert u_{S}\right\Vert _{2}\nonumber\\
&=-M\left(\lambda_{n}\sqrt{d}\right)^{2}.\label{eq:second-term-bound}
\end{align}

The remaining middle Hessian term in RHS of (\ref{eq:G-Taylor}) is,  different from \cite{ravikumar2010high}, quite simple due to the square loss function:
\begin{align}
u_{S}^{T}Q_{S}^{n}u_{S} & \geq\left\Vert u_{S}\right\Vert _{2}^{2}\Lambda_{\min}\left(Q_{SS}^{n}\right)\nonumber \\
 & \geq C_{\min}M^{2}\left(\lambda_{n}\sqrt{d}\right)^{2},\label{eq:third-term-bound}
\end{align}
where the last inequality comes from the dependency condition  $\Lambda_{\min}\left(Q_{SS}^{*}\right)\geq C_{\min}$
in (\ref{eq:eigen-upper-lower}). In contrast to \cite{ravikumar2010high},
there is no need to control the additional spectral norm. 

Combining the three bounds (\ref{eq:first-term-bound}) - (\ref{eq:third-term-bound}) together with (\ref{eq:G-Taylor}), we obtain that
\begin{equation}
G\left(u_{S}\right)\geq\left(\lambda_{n}\sqrt{d}\right)^{2}\left\{ -\frac{M}{2}+C_{\min}M^{2}-M\right\} .\label{eq:G-lower-bound}
\end{equation}
It can be easily verified from (\ref{eq:G-lower-bound}) that $G\left(u_{S}\right)$
is strictly positive when we choose $M=\frac{3}{C_{\min}}$. Consequently,
as long as $\left\Vert W^{n}\right\Vert _{\infty}\leq\frac{\lambda_{n}}{2}$,
we are guaranteed that $\left\Vert \hat{u}_{S}\right\Vert _{2}\leq M\lambda_{n}\sqrt{d}=\frac{3\lambda_{n}\sqrt{d}}{C_{\min}}$,
which completes the proof. 
\end{proof}

\subsection{Proof of Lemma \ref{lem:W_inf}\label{appendix:W_inf_convexity-proof}}
\begin{proof}
According to Lemma \ref{lem: Ws-Zs-results}, applying the
Bernstein's inequality, $\forall\eta>0$ we have 
\begin{equation}
\mathbb{P}\left(\left|W_{s}^{n}\right|>\eta\right)\leq2\exp\left(-\frac{\frac{1}{2}\eta^{2}n}{1+\frac{1}{3}d\eta}\right).\label{eq:Wns_inequality-2}
\end{equation}
Similar to \cite{vuffray2016interaction}, inverting the following
relation
\begin{equation}
\xi=\text{\ensuremath{\frac{\frac{1}{2}\eta^{2}n}{1+\frac{1}{3}d\eta}}}\label{eq:s-invert}
\end{equation}
and substituting the result in (\ref{eq:Wns_inequality-2}) yields
\begin{equation}
\mathbb{P}\left(\left|W_{s}^{n}\right|>\frac{1}{3}\left(u+\sqrt{u^{2}+18\frac{u}{d}}\right)\right)\leq2\exp\left(-\xi\right).\label{eq:Wns_inequality-1}
\end{equation}
where $u=\frac{\xi}{n}d$. Suppose that $n\geq\xi d^{2}$, then $u^{2}=\frac{\xi^{2}}{n^{2}}d^{2}\leq\frac{\xi}{n}$
while $\frac{u}{d}=\frac{\xi}{n}$. Consequently, we have
\begin{align}
\frac{1}{3}\left(u+\sqrt{u^{2}+18\frac{u}{d}}\right) & \leq\frac{1}{3}\left(\sqrt{\frac{\xi}{n}}+\sqrt{\frac{\xi}{n}+18\frac{\xi}{n}}\right)\\
 & \leq\frac{1}{3}\left(\sqrt{\frac{\xi}{n}}+\sqrt{\frac{\xi}{n}}\sqrt{25}\right)\\
 & =2\sqrt{\frac{\xi}{n}}.
\end{align}
where a relaxed result is obtained. Subsequently, we obtain an expression
which is independent of $d$
\begin{equation}
\mathbb{P}\left(\left|W_{s}^{n}\right|>2\sqrt{\frac{\xi}{n}}\right)\leq2\exp\left(-\xi\right).\label{eq:Wns_inequality-1-1}
\end{equation}
Then, by using a union bound we have
\begin{align}
\mathbb{P}\left(\left\Vert W^{n}\right\Vert _{\infty}>2\sqrt{\frac{\xi}{n}}\right) & \leq2\exp\left(-\xi+\log p\right).
\end{align}
Setting $\xi=\log\frac{2p}{\varepsilon_{3}}$, then if $n\geq d^{2}\log\frac{2p}{\varepsilon_{3}}$,
we have
\begin{align}
\mathbb{P}\left(\left\Vert W^{n}\right\Vert _{\infty}>2\sqrt{\frac{\log\frac{2p}{\varepsilon_{3}}}{n}}\right) & \leq2\exp\left(-\log\frac{2p}{\varepsilon_{3}}+\log p\right)\\
 & =\varepsilon_{3},
\end{align}
which completes the proof. 
\end{proof}

\subsection{Proof of Lemma \ref{lem:H-concentrate-result}\label{lem:H-concentrate-result-proof}}
\begin{proof}
Since $x_{r}^{(i)}x_{t}^{(i)}$ is bounded by $\left|x_{r}^{(i)}x_{t}^{(i)}\right|\leq1$.
Therefore, using the Hoeffding inequality \cite{hoeffding1994probability},
for any $\epsilon>0$, there is 
\begin{align}
\mathbb{P}\left(\left|Q_{st}^{n}-Q^{*}_{st}\right|>\epsilon\right) & \leq2\exp\left(-\frac{n\epsilon^{2}}{2}\right).
\end{align}
Then, due to the symmetry of $Q_{st}^{n}$, using a union bound we have
\begin{align}
\mathbb{P}\left(\left|Q_{st}^{n}-Q^{*}_{st}\right|\leq\epsilon,\forall s,t\in V\setminus r\right) & \geq1-p^{2}\exp\left(-\frac{n\epsilon^{2}}{2}\right),
\end{align}
As a result, as long as $n\geq\frac{2}{\epsilon^{2}}\log\frac{p^{2}}{\varepsilon_{4}}$, there is 
$\mathbb{P}\left(\left|Q_{st}^{n}-Q^{*}_{st}\right|\leq\epsilon,\forall s,t\in V\setminus r\right)\geq1-\varepsilon_{4}$, 
which completes the proof.
\end{proof}

\subsection{Proof of Lemma \ref{lem:strong-restricted}\label{appendix:strong-convexity-proof}}
\begin{proof}
According (\ref{eq:delta-loss-def}) and Lemma \ref{lem:H-bound-eign},
we have
\begin{align}
&\delta\ell\left(\triangle_{\theta_{\setminus r}},\tilde{\theta}_{\setminus r}^{*};\mathfrak{X}_{1}^{n}\right) \nonumber \\
&=\frac{1}{2}\triangle_{\theta_{\setminus r}}^{T}Q^{n}\triangle_{\theta_{\setminus r}}\nonumber \\
 & =\frac{1}{2}\triangle_{\theta_{\setminus r}}^{T}Q^{*}\triangle_{\theta_{\setminus r}}+\frac{1}{2}\triangle_{\theta_{\setminus r}}^{T}\left(Q^{n}-Q^{*}\right)\triangle_{\theta_{\setminus r}}\nonumber \\
 & \geq\frac{e^{-2\theta_{\max}^{*}d}}{2\left(d+1\right)}\left\Vert \triangle_{\theta_{\setminus r}}\right\Vert _{2}^{2}+\frac{1}{2}\triangle_{\theta_{\setminus r}}^{T}\left(Q^{n}-Q^{*}\right)\triangle_{\theta_{\setminus r}}.
\end{align}
Then, from Lemma \ref{lem:H-concentrate-result}, choosing $\epsilon=\frac{e^{-2\theta_{\max}^{*}d}}{32d\left(d+1\right)}$,
then with probability at least $1-\varepsilon_{4}$, there is
\begin{align}
\triangle_{\theta_{\setminus r}}^{T}\left(Q^{n}-Q^{*}\right)\triangle_{\theta_{\setminus r}} & \geq-\frac{e^{-2\theta_{\max}^{*}d}}{32d\left(d+1\right)}\left\Vert \triangle_{\theta_{\setminus r}}\right\Vert _{1}^{2}\nonumber \\
 & \geq-\frac{e^{-2\theta_{\max}^{*}d}}{2\left(d+1\right)}\left\Vert \triangle_{\theta_{\setminus r}}\right\Vert _{2}^{2}.
\end{align}
as long as $n\geq\frac{2}{\epsilon^{2}}\log\frac{p^{2}}{\varepsilon_{4}}=2^{11}d^{2}\left(d+1\right)^{2}e^{4\theta_{\max}^{*}d}\log\frac{p^{2}}{\varepsilon_{4}}$.
As a result, there is
\begin{align}
&\delta\ell\left(\triangle_{\theta_{\setminus r}},\tilde{\theta}_{\setminus r}^{*};\mathfrak{X}_{1}^{n}\right) \nonumber \\
& \geq\frac{e^{-2\theta_{\max}^{*}d}}{2\left(d+1\right)}\left\Vert \triangle_{\theta_{\setminus r}}\right\Vert _{2}^{2}-\frac{e^{-2\theta_{\max}^{*}d}}{4\left(d+1\right)}\left\Vert \triangle_{\theta_{\setminus r}}\right\Vert _{2}^{2}\nonumber \\
 & =\frac{e^{-2\theta_{\max}^{*}d}}{4\left(d+1\right)}\left\Vert \triangle_{\theta_{\setminus r}}\right\Vert _{2}^{2},
\end{align}
which completes the proof. 
\end{proof}

\section{\label{appendix-main-proof-no-thd} Proofs of Theorems \ref{theorem-RR-Lasso-no-post} and \ref{theorem-main-tree-graph-noPost}}
First, to prove the ``fixed design'' results in  Proposition \ref{theorem-main-fixed} and Proposition \ref{theorem-main-fixed-tree}, for each vertex $r\in V$, an optimal primal-dual pair
$\left(\hat{\theta}_{\setminus r},\hat{z}_{r}\right)$ is constructed,
where $\hat{\theta}_{\setminus r}\in\mathbb{R}^{p-1}$ is a primal
solution and $\hat{z}_{r}\in\mathbb{R}^{p-1}$ is the associated sub-gradient
vector. They satisfy the zero sub-gradient optimality condition \cite{rockafellar1970convex}
associated with  Lasso  (\ref{eq:lasso-estimator}): 
\begin{equation}
\nabla\ell\left(\hat{\theta}_{\setminus r};\mathfrak{X}_n\right)+\lambda_{n}\hat{z}_{r}=0,\label{eq:zero-gradient-cond}
\end{equation}
where the sub-gradient vector $\hat{z}_{r}$  satisfies
\begin{equation}
\begin{cases}
\hat{z}_{rt}=\textrm{sign}\left(\hat{\theta}_{rt}\right), \textrm{if }\hat{\theta}_{rt}\neq0; & \left(a\right)\\
\left|\hat{z}_{rt}\right|\leq1,\textrm{otherwise}. & \left(b\right)
\end{cases}\label{eq:zero-sub-gradient-z}
\end{equation}
Then, the pair
is a primal-dual optimal solution to (\ref{eq:lasso-estimator}) and
its dual. Further, to ensure that such an optimal primal-dual pair correctly
specifies the signed neighorbood of node $r$,  the sufficient
and necessary conditions  are as follows
\begin{equation}
\begin{cases}
\textrm{sign}\left(\hat{z}_{rt}\right)=\textrm{sign}\left(\theta_{rt}^{*}\right),\textrm{ }\forall\left(r,t\right)\in S, & \left(a\right)\\
\hat{\theta}_{ru}=0,\textrm{}\textrm{ }\forall\left(r,u\right)\in S^{c}\coloneqq E\setminus S. & \left(b\right)
\end{cases}\label{eq:zero-sub-gradient-z-optimal}
\end{equation}
Note that while the regression in (\ref{eq:lasso-estimator}) corresponds
to a convex problem, for $p\gg n$ in the high-dimensional regime,
it is not necessarily strictly convex so that there might be multiple
optimal solutions. Fortunately, the following lemma in \cite{ravikumar2010high}
provides sufficient conditions for shared sparsity among optimal solutions
as well as uniqueness of the optimal solution. 
\begin{lem}
\label{lem:uniquess-solution}(Lemma 1 in \cite{ravikumar2010high}).
Suppose that there exists an optimal primal solution $\hat{\theta}_{\setminus r}$
with associated optimal dual vector $\hat{z}_{r}$ such that $\left\Vert \hat{z}_{S^{c}}\right\Vert _{\infty}<1$.
Then any optimal primal solution $\tilde{\theta}$ must have $\tilde{\theta}_{S^{c}}=0$.
Moreover, if the Hessian sub-matrix $[ \nabla^2\ell\left(\hat{\theta}_{\setminus r};\mathfrak{X}_n\right)]_{SS}$
is strictly positive definite, then $\hat{\theta}_{\setminus r}$
is the unique optimal solution. 
\end{lem}
As a result, using the framework in \cite{ravikumar2010high}, 
we can construct a primal-dual witness $\left(\hat{\theta}_{\setminus r},\hat{z}\right)$
for the Lasso estimator (\ref{eq:lasso-estimator}) as follows:

(a) First, set $\hat{\theta}_{S}$ as the minimizer of the partial
penalized likelihood 
\begin{align}
\hat{\theta}_{S} & =\underset{\theta_{\setminus r}=\left(\theta_{S},0\right)\in\mathbb{R}^{p-1}}{\arg\min}\left\{ \ell\left(\theta_{\setminus r};\mathfrak{X}_n\right)+\lambda_{n}\left\Vert \theta_{S}\right\Vert _{1}\right\} ,\label{eq:lasso-estimator-partial}
\end{align}
and then set $\hat{z}_{S}=\textrm{sign}\left(\hat{\theta}_{S}\right)$. 

(b) Second, set $\hat{\theta}_{S^{c}}=0$ so that condition (\ref{eq:zero-sub-gradient-z-optimal})
(b) holds.

(c) Third, obtain $\hat{z}_{S^{c}}$ from (\ref{eq:zero-gradient-cond})
by substituting the values of $\hat{\theta}_{\setminus r}$ and $\hat{z}_{S}$.

(d) Finally, we need to show that the stated scalings of $\left(n,p,d\right)$
imply that, with high probability, the remaining conditions (\ref{eq:zero-sub-gradient-z}) and  (\ref{eq:zero-sub-gradient-z-optimal})
(a) are satisfied. 

\subsection{\label{appendix:proof-proposition-RR}Proof of Proposition \ref{theorem-main-fixed}}
From Lemma \ref{lem:W_inf_norm} (a),
if the regularization parameter $\lambda_{n}$ satisfies $\lambda_{n}\geq\frac{8\left(2-\alpha\right)}{\alpha}\sqrt{\frac{\log p}{n}}$, then with probability greater than $1-2\exp\left(-c\lambda_n^2n\right)$ there is 
\begin{equation}
\left\Vert W^{n}\right\Vert _{\infty}\leq\frac{\alpha}{2-\alpha}\frac{\lambda_{n}}{2}\leq\frac{\lambda_{n}}{2},\label{eq:wn-inf-inequality}
\end{equation}
so that the
condition in Lemma \ref{lem:L2-consistency} is also satisfied. The
zero-subgradient condition (\ref{eq:zero-gradient-cond})
can be equivalently re-written as follows
\begin{equation}
\begin{cases}
Q_{S^{c}S}^{n}\left(\hat{\theta}_{S}-\tilde{\theta}_{S}^{*}\right)=W_{S^{c}}^{n}-\lambda_{n}\hat{z}_{S^{c}},\\
Q_{SS}^{n}\left(\hat{\theta}_{S}-\tilde{\theta}_{S}^{*}\right)=W_{S}^{n}-\lambda_{n}\hat{z}_{S},
\end{cases}\label{eq:zero-mean-cond-block-form}
\end{equation}
where we have used the fact that $\hat{\theta}_{S^{c}}=0$ from the
primal-dual construction, and also the result $\tilde{\theta}^*_{S^{c}}=0$ from Lemma \ref{lemma-recaled-solution}. After some simple algebra, we obtain
\begin{equation}
W_{S^{c}}^{n}-Q_{S^{c}S}^{n}\left(Q_{SS}^{n}\right)^{-1}W_{S}^{n}+\lambda_{n}Q_{S^{c}S}^{n}\left(Q_{SS}^{n}\right)^{-1}\hat{z}_{S}=\lambda_{n}\hat{z}_{S^{c}}.\label{eq:zero-mean-condition-rearrage}
\end{equation}
For strict dual feasibility, from (\ref{eq:zero-mean-condition-rearrage}), we obtain
\begin{align}
\left\Vert \hat{z}_{S^{C}}\right\Vert _{\infty} & \leq|||Q_{S^{C}S}^{*}\left(Q_{SS}^{*}\right)^{-1}|||_{\infty}\left[\frac{\left\Vert W_{S}^{n}\right\Vert _{\infty}}{\lambda_{n}}+1\right]\nonumber \\
&+\frac{\left\Vert W_{S^{C}}^{n}\right\Vert _{\infty}}{\lambda_{n}}\nonumber \\
 & \leq\left(1-\alpha\right)+\left(2-\alpha\right)\frac{\left\Vert W^{n}\right\Vert _{\infty}}{\lambda_{n}}\nonumber \\
 & \leq\left(1-\alpha\right)+\left(2-\alpha\right)\frac{1}{2-\alpha}\frac{\alpha}{2}\nonumber \\
 & =1-\frac{\alpha}{2}<1\label{eq:zc}, 
\end{align}
with probability converging to one. 
For correct sign recovery, it suffices to show that $\left\Vert \hat{\theta}_{S}-\tilde{\theta}_{S}^{*}\right\Vert _{\infty}\leq\frac{\tilde{\theta}_{\min}^{*}}{2}$.
From Lemma \ref{lem:L2-consistency} (since (\ref{eq:wn-inf-inequality})
holds), we have
\begin{align}
\frac{2}{\theta_{\min}^{*}}\left\Vert \hat{\theta}_{S}-\tilde{\theta}_{S}^{*}\right\Vert _{\infty}  \leq\frac{2}{\theta_{\min}^{*}}\left\Vert \hat{\theta}_S-\tilde{\theta}_{S}^{*}\right\Vert _{2}
 \leq\frac{6}{\tilde{\theta}_{\min}^{*}C_{\min}}\lambda_{n}\sqrt{d}.
\end{align}
As a result, if $\tilde{\theta}_{\min}^{*}\geq\frac{6\lambda_{n}\sqrt{d}}{C_{\min}}$,
or $\lambda_{n}\leq\frac{\tilde{\theta}_{\min}^{*}C_{\min}}{6\sqrt{d}}$,
the condition $\left\Vert \hat{\theta}_{S}-\tilde{\theta}_{S}^{*}\right\Vert _{\infty}\leq\frac{\tilde{\theta}_{\min}^{*}}{2}$
holds. In the paramagnetic phase, from Lemma \ref{lemma-recaled-solution}, there is $\tilde{\theta}_{\min}^{*}=\frac{\tanh\left(\theta_{0}\right)}{1+\left(d-1\right)\tanh^{2}\left(\theta_{0}\right)}$. Substituting these results lead to Proposition \ref{theorem-main-fixed}.

\subsection{Proof of Proposition \ref{theorem-main-fixed-tree}}
\label{appendix:prop-tree-proof}
The proof of Proposition \ref{theorem-main-fixed-tree} is the same as that of Proposition \ref{theorem-main-fixed}
in Appendix \ref{appendix:proof-proposition-RR}, except that different conditions in Lemma \ref{lem:W_inf_norm} (b) are used, and that we need to impose the assumptions that the population Hessian $Q^*$ satisfies both conditions (C1) and (C2) for the considered general graphs.  

\subsection{\label{appendix:proof-theorem-1}Proof of Theorem \ref{theorem-RR-Lasso-no-post}}
Now we are ready to prove the main results in Theorem \ref{theorem-RR-Lasso-no-post}. 
As shown in Lemma \ref{lem:RR-graph-(C1)(C2)}, for RR graphs with uniform couplings, the population Hessian $Q^*$ for Lasso already satisfies both conditions (C1) and (C2), so that  assumptions of (C1) and (C2) can be dropped for RR graphs.  

Next,  using large deviation analysis as \cite{ravikumar2010high}, we prove that the sample Hessian  $Q^n$ of Lasso satisfies the same properties as the population Hessian  $Q^*$  with high probability with large enough samples. 
\begin{lem}
\label{lem:population-to-sample-Hessian} Consider an Ising model on a RR graph $G=\left(V,E\right)\in\mathcal{G}_{p,d}$ with regular  node degree $d$ and uniform couplings $\theta^{*}_{r,t} = \theta_0,\forall (r,t)\in E$. 
Then, for any $\delta>0$, there are  some positive constants $A,B,K$ 
\begin{align}
&\mathbb{P}\left(\Lambda_{\min}\left(Q^{n}_{SS}\right)\leq C_{\min}-\delta\right)  \leq2\exp\big(-A\frac{\delta^{2}n}{d^{2}}+B\log d\big),\label{eq:A3-lemma} \\
&\mathbb{P}\left(\interleave Q_{S^{c}S}^{n}\left(Q_{SS}^{n}\right)^{-1}\interleave_{\infty}\geq1-\frac{\alpha}{2}\right)  \leq2\exp\big(-K\frac{n}{d^{3}}+\log p\big),
\end{align}
where $C_{\min} $ and $\alpha$ are $C_{\min} = 1 - \tanh^{2}\theta_{0}$ and $\alpha = 1 - \tanh\theta_{0}$.
\end{lem}
\begin{proof}
The proof is the same as Lemma 5 and Lemma 6 in \cite{ravikumar2010high}, with the only difference that the variance function term does not exist, by substituting into the the results of $C_{\min} $ and $\alpha$ in the Lemma \ref{lem:RR-graph-(C1)(C2)}. 
\end{proof}

Lemma \ref{lem:population-to-sample-Hessian} demonstrates that the sample Hessian  $Q^n$ satisfies both conditions (C1) and (C2) with high probability as long as $n\geq Ld^3\log p$ for some constant $L$. As the results of Proposition \ref{theorem-main-fixed}  builds on top of the assumption that the sample Hessian $Q^n$ satisfies (C1) and (C2), we readily obtain that all results of Proposition \ref{theorem-main-fixed} will hold for if we replace the requirement that the sample Hessian  $Q^n$ satisfies both conditions (C1) and (C2) by an extra scaling requirement $n\geq Ld^3\log p$ for some constant $L$ independent of $\left(n,p,d\right)$.

Consequently, by combining Lemma \ref{lem:RR-graph-(C1)(C2)}, Lemma \ref{lem:population-to-sample-Hessian}, and Proposition \ref{theorem-main-fixed} and substituting the specific results of $C_{\min} $ and $\alpha$ in Lemma \ref{lem:RR-graph-(C1)(C2)}, after some algebra, we readily obtain Theorem \ref{theorem-RR-Lasso-no-post}, which completes the proof.

\subsection{Proof of Theorem \ref{theorem-main-tree-graph-noPost}}
The proof of Theorem \ref{theorem-main-tree-graph-noPost} is the same as that of Theorem \ref{theorem-RR-Lasso-no-post}
in Appendix \ref{appendix:proof-theorem-1}, except that different conditions in Lemma \ref{lem:W_inf_norm} (b) are used.

\section{\label{appendix-main-proof-with-thd}Proofs of Theorems \ref{theorem-square-error} and  \ref{theorem-Post-Thresholding}}
\subsection{Proofs of Theorem \ref{theorem-square-error}}
This is done through Proposition \ref{prop:square-error}
by evaluating the two conditions (C3) and (C4). First, let $\varepsilon_{3}=\frac{2\varepsilon_{1}}{3}>0$
in Lemma \ref{lem:W_inf}. Then, by setting $\lambda_{n}=4\sqrt{\frac{\log\frac{3p}{\varepsilon_{1}}}{n}}$,
if $n\geq d^{2}\log\frac{3p}{\varepsilon_{1}}$ , with probability
at least $1-\frac{2\varepsilon_{1}}{3}$, we have $\left\Vert W^{n}\right\Vert _{\infty}\leq2\sqrt{\frac{\log\frac{3p}{\varepsilon_{1}}}{n}}=\frac{\lambda_{n}}{2}$
so that condition (C3) satisfies as long as $n\geq d^{2}\log\frac{3p}{\varepsilon_{1}}$.
Second, let $\varepsilon_{4}=\frac{\varepsilon_{1}}{3}>0$ in Lemma
\ref{lem:strong-restricted}. From Lemma \ref{lem:strong-restricted},
with probability at least $1-\frac{\varepsilon_{1}}{3}$, the restricted
strong convexity condition is satisfied with the value $\kappa=\frac{e^{-2\theta_{\max}^{*}d}}{4\left(d+1\right)}$
when $n>2^{11}d^{2}\left(d+1\right)^{2}e^{4\theta_{\max}^{*}d}\log\frac{3p^{2}}{\varepsilon_{1}}$.
Then, the relation $R\geq3\sqrt{d}\frac{\lambda_{n}}{\kappa}$ in
Proposition \ref{prop:square-error} reads
\begin{equation}
R>3\sqrt{d}4\sqrt{\frac{\log\frac{3p}{\varepsilon_{1}}}{n}}\left(\frac{e^{-2\theta_{\max}^{*}d}}{4\left(d+1\right)}\right)^{-1}.\label{eq:R-result-cond}
\end{equation}
To find a value of $R$ that satisfies (\ref{eq:R-result-cond}),
we can choose $R=2/\sqrt{d}$. Then from (\ref{eq:R-result-cond}),
the number of samples $n$ needs to satisfy
\begin{equation}
n>9\cdot2^{10}d^{2}\left(d+1\right)^{2}e^{4\theta_{\max}^{*}d}\log\frac{3p^{2}}{\varepsilon_{1}}.\label{eq:n-result-cond1}
\end{equation}
As a result, when $n\geq2^{14}d^{2}\left(d+1\right)^{2}e^{4\theta_{\max}^{*}d}\log\frac{3p^{2}}{\varepsilon_{1}}$,
the condition  (C4) satisfies with probability at least $1-\frac{\varepsilon_{1}}{3}$.
Based on the union bound, both condition (C3) and condition  (C4)  will be
simultaneously satisfied with probability at least $1-\varepsilon_{1}$,
which completes the proof by using Proposition \ref{prop:square-error}. 

\subsection{Proofs of Theorem  \ref{theorem-Post-Thresholding}}
First consider any
fixed vertex $r\in V$, if the square error $\left\Vert \hat{\theta}_{\setminus r}-\tilde{\theta}_{\setminus r}^{*}\right\Vert _{2}\leq\frac{\tilde{\theta}_{\min}^{*}}{2}$,
then it is guaranteed that the absolute difference of each element
of $\hat{\theta}_{\setminus r}$ and $\tilde{\theta}_{\setminus r}^{*}$
is less than $\frac{\tilde{\theta}_{\min}^{*}}{2}$ so that one can
perfectly recover all its correct neighbors with a thresholding $\frac{\tilde{\theta}_{\min}^{*}}{2}$.
According to Theorem \ref{theorem-square-error}, with probability
$1-\varepsilon_{1}$, when $n\geq2^{14}d^{2}\left(d+1\right)^{2}e^{4\theta_{\max}^{*}d}\log\frac{3p^{2}}{\varepsilon_{1}}$,
there is $\left\Vert \hat{\theta}_{\setminus r}-\tilde{\theta}_{\setminus r}^{*}\right\Vert _{2}\leq2^{6}\sqrt{d}\left(d+1\right)e^{2\theta_{\max}^{*}d}\sqrt{\frac{\log\frac{3p}{\varepsilon_{1}}}{n}}$.
Further, let $2^{6}\sqrt{d}\left(d+1\right)e^{2\theta_{\max}^{*}d}\sqrt{\frac{\log\frac{3p}{\varepsilon_{1}}}{n}}\leq\frac{\tilde{\theta}_{\min}^{*}}{2}$,
we obtain that $n\geq2^{14}\left(\tilde{\theta}_{\min}^{*}\right)^{-2}d\left(d+1\right)^{2}e^{4\theta_{\max}^{*}d}\log\frac{3p}{\varepsilon_{1}}$.
Consequently, with at least probability $1-\varepsilon_{1}$ we have
$\left\Vert \hat{\theta}_{\setminus r}-\tilde{\theta}_{\setminus r}^{*}\right\Vert _{2}\leq\frac{\tilde{\theta}_{\min}^{*}}{2}$
and thus correct neighbors are recovered for any fixed $r\in V$ whenever
\begin{equation}
n\geq\max\left\{ d,\left(\tilde{\theta}_{\min}^{*}\right)^{-2}\right\} 2^{14}d\left(d+1\right)^{2}e^{4\theta_{\max}^{*}d}\log\frac{3p^{2}}{\varepsilon_{1}}.\label{eq:n-for-structure}
\end{equation}
Then, setting $\varepsilon_{2}=p\varepsilon_{1}$ and using the union
bound for all vertices $r\in V$, we have
\begin{equation}
\mathbb{P}\left(\left\Vert \hat{\theta}_{\setminus r}-\tilde{\theta}_{\setminus r}^{*}\right\Vert _{2}>\frac{\tilde{\theta}_{\min}^{*}}{2},\exists\;r\in V\right)\leq p\varepsilon_{1}=\varepsilon_{2},
\end{equation}
so that
\begin{equation}
\mathbb{P}\left(\left\Vert \hat{\theta}_{\setminus r}-\tilde{\theta}_{\setminus r}^{*}\right\Vert _{2}\leq\frac{\tilde{\theta}_{\min}^{*}}{2},\forall\;r\in V\right)>1-\varepsilon_{2},
\end{equation}
which completes the proof.

\vfill

\end{document}